\documentclass{article} 

\usepackage{mathrsfs}
\usepackage{url}
\usepackage{amsmath,amsfonts,amssymb,enumerate}

\usepackage{algorithmic,algorithm}
\usepackage{arydshln}
\usepackage{extarrows}
 \usepackage{epsfig}
\usepackage{mathrsfs}
\usepackage{amsfonts,amsmath}
\usepackage{graphicx}
\usepackage{subfigure}
\usepackage{setspace}
\usepackage[english]{babel}
\usepackage{colordvi,color}

\newtheorem{example}{Example}
\newtheorem{alg}{Algorithm}
\newtheorem{remark}{Remark}
\DeclareMathOperator*{\argmin}{arg\,min}

\newcommand{\diver}{\text{div}}
\newcommand{\dd}{\boldsymbol{d}}
\newcommand{\pp}{\boldsymbol{p}}

\newcommand{\phI}[0]{ \boldsymbol{\phi} }

\newcommand{\newBlue}[1]{ { #1}}
    \newtheorem{lemma}{Lemma} 
    \newtheorem{theorem}{Theorem} 
    \newtheorem{proposition}{Proposition} 
    \newtheorem{definition}{Definition}  
   \newenvironment{proof}[1][Proof.]{\begin{trivlist}  
       \item[\hskip \labelsep {\bfseries #1}]}{\end{trivlist}}
\date{}
\title{A Total Fractional-Order Variation Model for Image Restoration with Non-homogeneous Boundary Conditions and its  Numerical Solution\thanks{This work was
supported by the UK EPSRC grant (number  EP/K036939/1) and the National Natural Science Foundation of China (NSFC
Project number 11301447).}}

\author{Jianping Zhang\thanks{Department of
              Mathematical Sciences, The University of Liverpool, United Kingdom and School of Mathematics and Computational Science, Xiangtan University, Xiangtan, Hunan 411105,
P. R. China.} and Ke Chen\thanks{Centre for Mathematical Imaging Techniques and Department of
              Mathematical Sciences, The University of Liverpool, United Kingdom. Email  k.chen@liv.ac.uk, Web
              {\bf www.liv.ac.uk/cmit}}  }

\begin{document}
\maketitle

\begin{abstract}
To overcome the weakness of a total variation based model for image restoration,
various high order (typically second order) regularization models have been proposed and studied recently. In this paper we analyze and test a
fractional-order derivative based total $\alpha$-order variation model,
which can outperform the currently popular high order regularization models.
There exist several previous works using total $\alpha$-order variations for image restoration; however first no analysis  is done yet and second
all tested formulations, differing from each other, utilize the zero Dirichlet boundary conditions
which are not realistic (while non-zero boundary conditions violate
definitions of fractional-order derivatives).

This paper first reviews some results of fractional-order derivatives
and then analyzes the theoretical properties
 of the proposed total $\alpha$-order variational  model rigourously.
 It then develops four algorithms for solving the variational problem, one based on the variational Split-Bregman idea and
 three based on direct solution of the discretise-optimization problem.
{Numerical experiments show that, in terms of  restoration  quality and solution efficiency, the proposed model can  produce
 highly competitive results, for smooth images, to two established high order models: the mean curvature and the total generalized variation}.

{\bf Keywords}.
Fractional-order derivatives; Total $\alpha$-order variation; PDE; Image Denoising;
 Image inverse problems; Optimization methods.
\ \
AMS.
62H35, 65N22, 65N55, 74G65, 74G75
\end{abstract}

\pagestyle{myheadings}
\thispagestyle{plain}
\markboth{Total $\alpha$-Order Variation Image Denoising}{Jianping Zhang and Ke Chen}

\section{Introduction}
This paper presents a fractional-order derivative based regularizer for variational image restoration. It may be used for other imaging models such as image registration.
Denote an observed image by $z=z(x)$, $x\in \Omega \subset \mathbb{R}^d$
where $\Omega$ is the bounded domain of the image with $d$ space dimension and has a Lipschitz boundary.
Here we consider $d=2$ and mainly the image denoising problem with an additive noise i.e. assume
$z=u + \eta_0$ with $\eta_0$ representing some unknown Gaussian noise of mean zero and deviation $\sigma$,
but most results are applicable to $d>2$ and other noise models.

\subsection{Image inverse problem}
Restoring the unknown $u$ (without any restrictions) from $z$   is an inverse problem.
According to the maximum likelihood principle \cite{SGeman1984}, most image processing problems involve solving the least-square problem
\begin{equation}\label{eq1.1}
\min_u\int_\Omega|P(u)-z|^2dx,
\end{equation}
measuring the fidelity to $z$.
For example, $P(u)=u$ for image denoising, $P(u)$ takes the template image $T(x+u(x))$ (and $z=R(x)$ for a reference image) for image registration, and $P(u)=P_{\Omega_1}(u(x))$ for image inpainting with $\Omega_1\subset
\Omega$ the subdomain with missing data.

 The problem (\ref{eq1.1}) is in general  ill-posed due to non-uniqueness, therefore how to effectively solve it becomes a fundamental task in image sciences.
The most popular idea is to regularize it so that the resulting well-posed problem admits an unique solution. The classical   regularization technique  by Tikhonov et. al \cite{ANTikhonov1977} is to add a smoothing regularization term into the energy functional to derive the following minimization problem
\begin{equation}
\min_u\int_\Omega|P(u)-z|^2dx+\lambda\int_\Omega|\nabla u|^2dx,
\end{equation}\label{eq1.2}
where $\lambda$ is a positive constant. This model cannot preserve image edges, though it is simple to use. The total variation (TV) model by Rudin-Osher-Fatemi \cite{LIRudin1992} or the ROF model
 \begin{equation}
\min_u\int_\Omega|\nabla u|dx,\qquad
\int_\Omega|P(u)-z|^2dx=\sigma^2,\ \ P(u)=u
\end{equation}\label{eq1.3}
is widely used, where $\sigma$ is an estimate of the error $\eta_0$ between the noisy image $z$ and the true data $u$. The ROF model preserves the image edges by seeking solutions of piecewise constant functions in the space of bounded variation functions
(BV). A variety of methods based on the TV regularization  have been developed to deal with the imaging problems such as image restoration \cite{RAcar1994,VAgarwal2007,JFAujol2009,JPZhang2012}, image registration \cite{LHomke2007,CFrohn-Schauf2008,TPock2007}, image decomposition \cite{SOsher2003,JBGarnett2007,VDuval2009}, image inpainting \cite{WGuo2007,PGetreuer2012,SGhate2012,TFChan2005} and image segmentation \cite{XBresson2007,MUnger2008}.
Restoring smooth images in some applications where edges are not the main features presents difficulties for the ROF model  as it can yield the so-called blocky (staircase) effects. Another disadvantage of the model is to the loss of image contrasts \cite{MLysaker2003}. It should be remarked that
the recently popular method by the iterative regularization technique
\cite{osher2005iterative} can reduce the staircasing effect and improve on the image contrast to some extent; besides it provides a fast implementation.

\subsection{High-order regularization}
To remedy the above mentioned two drawbacks (stairicasing and contrast),
 two types of alternative regularizer to the TV have been proposed in the literature. The first type introduces higher order regularization into image variational models \cite{TFChan2000,JShen2003,LAmbrosio2003,MLysaker2003,GSteidl2005,GDal-Maso2009,KBredies2010,WZhu2012}. The mean curvature-based variation denoising model was studied in \cite{MLysaker2003,MLysaker2004,Brito10,WZhu2012} where the regularized solution $u$ is obtained by solving the fourth-order Euler-Lagrangian equation.  Bredies et al. \cite{KBredies2010}   proposed the total generalized variation regularizer involving a linear combination of higher-order derivatives and the TV of $u$ to model the image denoising while
Chang et al. \cite{QSChang2009} considered a nonlinear combination of
regularizer based on
first and second order derivatives.
  For image inpainting, a high order regularization based on Euler's elastica of $u$ is used  in \cite{JShen2003}. Similarly the Euler's elastica energy \cite{JModersitzki2004} and mean curvature \cite{BFischer2003,Chum11} are also proposed to transform the template image $T(x+u)$ to map the
   reference image $R(x)$ in image registration; see also \cite{BFischer2002,HKostler2008}. The above mentioned high order regularization methods are effective but due to high nonlinearity efficient numerical solution is a major issue.

The second type introduces  fractional-order derivatives,
which are widely studied in other research subjects beyond image processing
\cite{Agrawal2002,AAlmeida2009,RAlmeida2011,AAtangana2013,ZYZhu2003}, into regularization of images.
For example, Bai and Feng   \cite{JBai2007} introduced first fractional-order
derivative into anisotropic diffusion equations for noise removal
\begin{equation}
\frac{\partial u}{\partial t}=-D^{\alpha *}_ x(c(|D^{\alpha}u|)D^{\alpha}u)-D^{\alpha *}_ y(c(|D^{\alpha}u|)D^{\alpha}u),
\end{equation}
where $c(\cdot)$ denotes the divergence parameter and $D_x^{\alpha *}$ denotes the adjoint operator of $D_x^{\alpha}$, which may be viewed as a generalization of the Perona-Malik model.
Although the  above equation
can be related to the Euler-Lagrange equations of an energy functional with the fractional derivative of the image intensity, generalizing
 commonly used PDE models, the energy minimization models are not studied as such. The discrete Fourier transform
is used to implement the numerical algorithm assuming a periodic  input image at its borders \cite{JBai2007}.
 See also \cite{PGuidotti2009a,PGuidotti2009b,MJanev2011,PDRomero2008} for more motivations and studies based on the above diffusion equation.
Chen et al.   \cite{DLChen2013a,DLChen2013b,DLChen2013c} considered the fractional-order TV-$L^2$ image denoising model
\begin{equation}\label{TV_alpha}
\min\limits_{u } \Big\{E(u):=
\int_\Omega \sqrt{(D^\alpha_xu)^2 + (D^\alpha_yu)^2}d\Omega
+\frac{\lambda}{2} \|u-f\|^2_2\Big\}
\end{equation}
and numerically obtained improved denoising results over
the Perona-Malik and ROF models; however no analysis was given. There, they converted this primal formulation into a dual problem for the new dual variable ${\mathbf p}=(p_1,p_2)$ by $u=f-\mbox{div}^\alpha{\mathbf p}/\lambda$
and
used a dual algorithm using the gradient descent idea
similar to the Chambolle method \cite{AChambolle2004} for the ROF.
In \cite{JZhang2012}, the authors proposed a {\it discrete} optimization framework for image denoising problem where the fractional order derivative is used to model the regularization term,
\begin{equation}
\min_u\Big\{ \sum_{i,j=1}^N|(\nabla^\alpha u)_{i,j}|+1/2\sum_{j=0}^L2^{-2js_j}|[\lambda(f-u)_j]|^2,\;1\leq\alpha\leq 2,0\leq s_j\leq 1\Big\},
\end{equation}
which  is solved by an alternating projection algorithm.
See also \cite{Ray13}.

These works have reflected good performance of the fractional order derivative in achieving a satisfactory compromise such as no stair-casing and in preserving important fine-scale features such as edges and textures.
These encouraging results motivated us to investigate this new model more closely.

There have been several other
works involving discrete forms of
 an $\alpha$-order derivative proposed to tackle
image registration problem \cite{RVerdu2009,AMelbourne2012} and image inpainting problem \cite{YZhang2012}. Comparing with the first type of high order models \cite{KBredies2010,BFischer2003,Chum11},
a fractional order model (type two) is less nonlinear and hence is more amenable to developing fast iterative solvers.
Clearly there is strong evidence to suggest that fractional order derivatives
may be effective regularizer for imaging applications.
There is an urgent need to establish a rigorous theory for   the total $\alpha$-order variation based variational model so that further
 applications to image inverse problems can be considered in a systematic way.

\subsection{Our contributions}
  This work is substantially different from previous studies. We mainly focus on the continuous
  total $\alpha$ variation-based model, instead of discrete formulation, and its analysis and associated numerical algorithms.

  Our contributions are four-fold:
  \begin{itemize}
  \item We analyze  properties of the total $\alpha$-order variation laying foundations for applications to image inverse problems as a regulariser;
  \item We give a new method for treating non-zero Dirichlet boundary conditions which represents a generalization of similar results that existed only in 1D to 2D;
  \item We establish  the convexity, the solvability  and a solution theory for the total $\alpha$-order variation model to make it more advantageous to work with than high order and  non-convex counterparts (such as a mean curvature based model) which are not gradient based and do  not have much known theory on their solutions;
  \item  We propose and test four solution algorithms (respectively Split-Bregman based, forward-backward algorithm, Nesterov accelerated method and  fast iterative shrinkage-thresholding algorithm -- FISTA) to solve the underlying
      total $\alpha$-order variation  model. We also compare with related models.
  \end{itemize}
 Our work is hoped to motivate further studies and facilitate  future applications of $\alpha$-order variation based regularizer to other imaging problems in the community.

 The rest of the paper is organized as follows. Section 2 reviews the definitions and basic properties of the fractional order derivative.  Section 3  first defines  the total $\alpha$-order variation   and the space
of functions of fractional-order bounded variations. In this space, it then analyzes the the existence and the uniqueness of the solution  of the total $\alpha$-order variation based model for denoising.
In Section 4, a boundary condition regularization method for treating nonzero Dirichlet boundary conditions is proposed to effectively employ and compute the fractional order derivatives of an image.
Section 5 first discusses the discretization of the fractional order derivatives
by a finite difference method and presents a Split-Bregman scheme for effective solution.
 Section 6 takes the alternative discretise-optimize solution approach and
  develops three optimization-based algorithms (Forward-backward algorithm, Nesterov accelerated method and FISTA)  to solve the image denoising model.
Experimental results are shown in Section 7, and the paper is concluded with a summary in Section 8.


\section{Review of fractional-order derivatives}
This section reviews definitions and simple properties of a fractional order derivative which   has a long history and may be considered as a generalization of the integer order derivatives.
Three popular definitions  to be reviewed are the Riemann-Liouville (R-L),  the Gr{\"u}nwald-Letnikov (G-L) and the Caputo definitions \cite{KSmiller1993,Oldham2006,IPodlubny1999}.

 In this paper, a fraction  $\alpha\in {\mathbb R}^+$ is assumed to lie in
 between two integers $n-1, n$ i.e. $0\leq \ell=n-1<\alpha<n$ and
  a fractional $\alpha$-order differentiation at point $x\in {\mathbb R}$ is denoted by the
  differential operator $D_{[a,x]}^\alpha$, where $a$ and $x$ are the bounds of the integral over a 1D computational domain.
   Undoubtedly, the gamma function is very important for the study of fractional derivative,
 which is defined by the integral \cite{IPodlubny1999}
\[\Gamma(z)=\int_0^\infty e^{-t}t^{z-1}\;dt.\]
One of the basic properties  is that  $\Gamma(z+1)=z\Gamma(z)$ and hence
$\Gamma(n)=n!$.
 Before introducing formal definitions, we review the following informative but classical example:
 \begin{example}
The Abel's integral equation, with,
\begin{equation}\label{example1-1}
\frac{1}{\Gamma(\alpha)}\int_0^x\frac{\psi(\tau)}{(x-\tau)^{1-\alpha}}d\tau=f(x),
\quad x>0
\end{equation}
has the solution  given by the well-known formula
\begin{equation}\label{example1-2}
\psi(x)=\frac{1}{\Gamma(1-\alpha)}\frac{d}{dx}\int_0^x\frac{f(\tau)}{(x-\tau)^\alpha}
d\tau,\quad x>0.
\end{equation}
\end{example}
This example helps to understand the formal definitions of fractional derivatives.
In fact for $0<\alpha<1$, equation (\ref{example1-1}) taking on the form
$I^{\alpha}_{[0,x]}\psi(x):=D^{-\alpha}_{[0,x]}\psi(x)=f(x)$
is called the fractional $\alpha$-order left R-L {\bf integral} of $\psi(x)$,
 and equation (\ref{example1-2}) taking on the form $D^{\alpha}_{[0,x]}f(x)=\psi(x)$ is defined as the fractional $\alpha$-order left R-L {\bf derivative} of $f(x)$.
 As operators, under suitable conditions \cite{samko1993fractional}, we have $D^{-\alpha}_{[0,x]}D^{\alpha}_{[0,x]}=I$ where $I$ denotes the identity operator.

The first definition of a general order $\alpha$ derivative  is the left sided R-L derivative
\begin{equation}
D^\alpha_{[a,x]}f(x)=\frac{1}{\Gamma(n-\alpha)}
\left(\frac{d}{dx}\right)^n
\int^x_a\frac{f(\tau)d\tau}{(x-\tau)^{\alpha-n+1}}.
\end{equation}
Subsequently the right-sided R-L
and the Riesz-R-L (central) fractional derivative are respectively given by
\[ \qquad
D^\alpha_{[x,b]}f(x)=\frac{(-1)^n}{\Gamma(n-\alpha)}\left(\frac{d}{dx}\right)^n
\int^b_x\frac{f(\tau)d\tau}{(\tau-x)^{\alpha-n+1}}\]
and\[
 D^\alpha_{[a,b]}f(x)=\frac{1}{2}\left(D^\alpha_{[a,x]}f(x)+(-1)^nD^\alpha_{[x,b]}f(x)\right).\]
The second  definition is the G-L left-sided derivative denoted by
\begin{equation}
{}^GD^\alpha_{[a,x]}f(x)=\lim\limits_{h\rightarrow 0}\frac{1}{h^\alpha}\sum_{j=0}^{[\frac{x-a}{h}]}(-1)^j\Bigg(\begin{array}{c}\alpha \\ j \\ \end{array}\Bigg)
f(x-jh),\;\;\;\Bigg(\begin{array}{c}\alpha \\ j \\ \end{array}\Bigg)=\frac{\alpha(\alpha-1)\dots(\alpha-j+1)}{j!},
\end{equation}
which resembles the definition for an integer order derivative, where $[\vartheta]$ is the integer such that
$\vartheta-1<[\vartheta]\leq \vartheta$.
The third definition is the Caputo order $\alpha$ derivative defined by
\begin{equation}\label{Caputo}
{}^{C}D^\alpha_{[a,x]}f(x)=\frac{1}{\Gamma(n-\alpha)}
  \int^x_a\frac{f^{(n)}(\tau)d\tau}{(x-\tau)^{\alpha-n+1}}.
\end{equation}
where $f^{(n)}$ denotes the $n^{th}$-order derivative of function $f(x)$.
The right sided derivative and the Riesz-Caputo fractional derivative are similarly defined by
\[ 
{}^{C}D^\alpha_{[x,b]}f(x)=\frac{(-1)^n}{\Gamma(n-\alpha)}
\int^b_x\frac{f^{(n)}(\tau)d\tau}{(\tau-x)^{\alpha-n+1}},\quad
{}^{C}D^\alpha_{[a,b]}f(x)=\frac{1}{2}\left({}^{C}D^\alpha_{[a,x]}f(x)
+(-1)^n{}^{C}D^\alpha_{[x,b]}f(x)\right).
\] 
When $\alpha=n-1$ is an integer, the above left-sided R-L definition reduces to the usual
 definition for a derivative.
{
One notes that when a function is $n-1$ times continuously differentiable and
its $n$th derivative is integrable,
the fractional derivatives by the above definitions are equivalent subject to homogeneous boundary conditions \cite{IPodlubny1999}. However we do not require such equivalence for our image function $u$; refer to Remark \ref{rem2} later.}

Fractional derivatives have many interesting properties ---  below we review
a few that are useful to this work.

 {\bf Linearity}. For a fractional derivative $D^\alpha_{[a,x]}$ by any of the above three definitions, then one has
\[D^\alpha_{[a,x]}(p\; f(x)+q\; g(x))=p\; D^\alpha_{[a,x]}f(x)+q\; D^\alpha_{[a,x]}g(x),\]
 for any fractional differentiable functions $f(x), g(x)$ and $p,q\in \mathbb{R}$.
This property will be shortly used to prove convexity and to derive the first order optimal conditions.

{\bf Zero fractional derivatives}.  An integer derivative of an image $u$ at pixels of flat regions may be close to zero but the left R-L derivative of a constant intensity function is not zero. One advantage of minimizing a R-L derivative instead of the total variation (image gradients) could be a non-constant solution.
It would be interesting to know the kind of functions that have zero $\alpha$-order derivatives.
\begin{lemma}[Singularity]\label{lemma_sigularity}
Assume that $D^\alpha_{[a,x]}$ is one of the above three fractional-order derivative operators. For any non-integer $\alpha>0$ and $x>a$, there exists a non-constant value function $f(\tau)$ in $(a,x]$ such that
\(D^\alpha_{[a,x]} f(x)=0.\)
\end{lemma}
\begin{proof}\label{Riemann-LiouvilleDefinition}
We give explicit constructions.
Here we only consider the R-L and Caputo derivatives; for G-L derivative, we can derive a similar conclusion through their equivalency.
\begin{enumerate}
\item Assume that $0<\alpha<1$, for some $x>0$, if $f(\tau)$ is taken as
\[f(\tau)=(x-2\tau)(x-\tau)^\alpha\]
for any $\tau\in(0,x]$ in Abel's inverse transform (\ref{example1-2}), then $D^\alpha_{[0,x]} f(x)=\psi(x)=0$;
\item Assume that $\alpha>1$, if $f(\tau)$ is taken as
\[f(\tau)=(x-\tau)^{\alpha-1}\]
for any $\tau\in(a,x]$ in $\alpha$-order R-L derivative, then $D^\alpha_{[a,x]} f(x)=0$;

\item Assume that $\alpha>0$ in Caputo derivative definition, if $f(\tau)$ is taken as
\[f(\tau)=(x-\tau)^{n-1}\]
for any $\tau\in(a,x]$ in equation (\ref{Caputo}), then ${}^CD^\alpha_{[a,x]} f(x)=0.$
\end{enumerate}
Actually for any $\alpha>0$, the left R-L $D_{[0,x]}^\alpha f(x)=0$ if $f(x)=x^{\alpha-k}$ for all $k=1,2,\dots,1+\ell$ (note $\ell=[\alpha]=n-1$);
 refer to \cite{hilfer2000applications}.
\end{proof}
\begin{remark}
For our later applications \S\ref{sec_num},
 we take $1<\alpha <2$. Hence we have
the left R-L $D_{[0,x]}^\alpha f(x)=0$ if $f(x)=x^{\alpha-1}$ or $x^{\alpha-2}$
i.e. $f(x)=x^{0.6}$ or $x^{-0.4}$ when $\alpha=1.6$.
For the Caputo derivative, ${}^CD_{[0,x]}^\alpha f(x)=0$ if $f(x)=1$ or $x$.
\end{remark}


%

{\bf Boundary conditions}. For the left R-L derivative $D^\alpha_{[a,x]}f(x)$
of $f(x)$, one assumes that $f(a)=0$ or $f(b)=0$ for the right R-L derivative; otherwise there is a singularity at the end point. So the Riesz R-L
derivative would require $f(a)=f(b)=0$. One solution for nonzero Dirichlet
boundary conditions for $f$ would be to extract off a linear approximation $g(x)$ (that
coincides with $f$ at $x=a,b$) and to consider $D^\alpha_{[a,x]}(f(x)-g(x))$;
however there was no such a method for the 2D case.
In Jumarie's work \cite{GJumarie2006}, a simple alternative is to modify the
R-L derivative to the following
\[D^\alpha_{[a,x]}f(x)=\frac{1}{\Gamma(n-\alpha)}
\left(\frac{d}{dx}\right)^n\int^x_a\frac{f(\tau)-f(a)}
{(x-\tau)^{\alpha-n+1}}d\tau,
\]
also
ensuring that the new fractional derivative of a constant is equal to zero and
removing the singularity at $x=a$ \cite{AAtangana2013}. In Section
\ref{sec_bndy}, we present one method for treating nonzero
Dirichlet boundary conditions in 2D.

\section{The total $\alpha$-order variation and its related model}
This section first studies the properties of the total $\alpha$-order variation, second analyzes a total $\alpha$-order variation based denoising model
and finally presents a numerical algorithm.
For the classical total variation based model, its solution lies in
a suitable space called
 the function space $\text{BV}(\Omega)$ of bounded variation
 \cite{TFChan2005b,OScherzer2009,KBredies2010}. From tests,
 the total fractional-order variation model can preserve both edges and
 smoothness of an image; we anticipate from the former that its solution should lie in a
 space
 similar to the BV space and from the latter that the smoothness is due to the non-local nature of the new regulariser.

 It turns out that for total $\alpha$-order variation using $\alpha$-order
 derivatives,
 a suitable space is
 the space $\text{BV}^\alpha(\Omega)$ of functions of $\alpha$-bounded variation on $\Omega$ which will be defined and studied next.
 The work of this section is motivated by analysis of the total variation (TV) \cite{RAcar1994,GAubert2006, AChambolle2004} and of the total generalized variation (TGV) \cite{KBredies2010}.

In variational regularization methods, integration by parts involves
the space of test functions in addition to the main solution space.
 Before discussing the total $\alpha$-order variation, we  give the following definition:
\begin{definition}[A space of test functions]
 Let $\mathcal{C}^\ell(\Omega,\mathbb{R}^d)$ denote the space of $\ell$-order continuously differentiable functions. Furthermore for any $\mathcal{C}^\ell(\Omega,\mathbb{R}^d)\ni v: \Omega\mapsto \mathbb{R}^d$, if the $(\ell+1){th}$ order derivative $v^{(\ell+1)}$ is integrable and $\frac{\partial^iv(x)}{\partial n^i}|_{\partial\Omega}=0$ for all $i=0,1,\dots,\ell$,   $v$ is   compactly supported continuous-integrable function in $\Omega$. Therefore the $\ell$-compactly supported continuous-integrable function space is denoted by
$\mathscr{C}_0^\ell(\Omega,\mathbb{R}^d)$.
\end{definition}

\begin{definition}[Total $\alpha$-order variation]\label{definition3.2}
Let $K$ denote  the space of
special test functions
\[
\begin{split}
K:=\Big\{\phI\in \mathscr{C}_0^\ell(\Omega,\mathbb{R}^d)\;\Big|\;|\phI(x)|\leq 1 \;\text{for all}\;x\in\Omega\Big\}
\end{split}
\]
where $|\phI|=\sqrt{\sum_{i=1}^d \phi_i^2}$.
Then the total $\alpha$-order variation of $u$ is defined by
\[
\text{TV}^\alpha(u):=\sup\limits_{\phI\in K}\int_\Omega\Big(-u\;\diver^\alpha\; \phI\Big)dx,
\]
 where $\diver^\alpha \phI=\sum_{i=1}^d \frac{\partial^\alpha \phi_i}{\partial x_i^\alpha} $ and $\frac{\partial^\alpha \phi_i}{\partial x_i^\alpha}$ denotes a fractional $\alpha$-order  {derivative} $D^\alpha_{[a,b]}\phi_i$ of $\phi_i$ along $x_i$ direction.
\end{definition}

{We note that $\text{TV}^\alpha(u)$ is the same for any definition of
$\frac{\partial^\alpha \phi_i}{\partial x_i^\alpha}$ because
$\phI$ satisfies the equivalence conditions. However for our applications in the paper, $\frac{\partial^\alpha u}{\partial x_i^\alpha}$ is generally not the same for different fractional derivatives (not even in the distributional sense). }

Based on the $\alpha$-BV semi-norm, the $\alpha$-BV norm is defined by
\begin{equation}
\|u\|_{\text{BV}^\alpha}=\|u\|_{L^1}+\text{TV}^\alpha(u),
\end{equation}
and further the space of functions of $\alpha$-bounded variation on $\Omega$ can be defined by
\begin{equation}\label{BV_alpha}
\text{BV}^\alpha(\Omega):=\Big\{u\in L^1(\Omega)\ \ \big| \ \ \text{TV}^\alpha(u)<+\infty\Big\}.
\end{equation}

\begin{lemma}[Lower semi-continuity]\label{lemma3-2}
Let $\{u^k(x)\}$ be a sequence from $\text{BV}^\alpha(\Omega)$ which converge in $L^1(\Omega)$ to a function $u(x)$. Then
\(\text{TV}^\alpha(u)\leq \liminf\limits_{k\rightarrow\infty}\text{TV}^\alpha(u^k).\)
\end{lemma}
\begin{proof}
Since $ \ u_k\in BV^\alpha(\Omega)$,
for any $\phI(x)\in\mathscr{C}_0^\ell(\Omega,\mathbb{R}^d)$ such that $|\phI(x)|\leq 1$ on $\Omega$,
 then $\diver^\alpha\; \phI$ is bounded, hence 
    \[\begin{split}
    \int_\Omega\Big(-u\;\diver^\alpha\; \phI\Big)dx&=\liminf\limits_{k\rightarrow +\infty}\int_\Omega\Big(-u_k\;\diver^\alpha\; \phI\Big)dx
    \leq \liminf\limits_{k\rightarrow +\infty}\text{TV}^\alpha(u_k)
    \end{split}\]
 from 
 $u_k\rightarrow u$ in $L^1(\Omega)$. Taking $\sup_{\phI(x)}$ in the  above inequality, we have lower semi-continuity from $\text{TV}^\alpha(u)\leq \liminf\limits_{k\rightarrow +\infty}\text{TV}^\alpha(u_k)$ (\emph{see} \cite{RAcar1994,evans1991measure} for TV case).
\end{proof}

\begin{lemma}\label{lemma3-3}
The space $\text{BV}^\alpha(\Omega)$   is a Banach space.
\end{lemma}
\begin{proof}
First we can see that $\text{BV}^\alpha(\Omega)$ is a normed space following immediately from the definitions of $\|u\|_{L^1(\Omega)}$ and total $\alpha$-order variation $\text{TV}^\alpha(u)$, so it only remains to prove completeness. Suppose $\{u^k\}$ is a Cauchy sequence in $\text{BV}^\alpha(\Omega)$; then, by the definition of the norm, it must also be a Cauchy sequence in $L^1(\Omega)$. According to the completeness of $L^1(\Omega)$, there exists a function $u$ in $L^1(\Omega)$ such that $u^k\rightarrow u$ in $L^1(\Omega)$.

Since $\{u^k\}$ is a Cauchy sequence in $\text{BV}^\alpha(\Omega)$, $\|u^k\|_{\text{BV}^\alpha}$ is bounded. Thus $\text{TV}^\alpha(u^k)$ is bounded as $k\rightarrow\infty$, by the lower semi-continuity
 of $\text{TV}^\alpha(u)$ in $\text{BV}^\alpha(\Omega)$ space
  (\emph{see} Lemma \ref{lemma3-2}), one shows that $u\in\text{BV}^\alpha(\Omega)$.

We shall show that $u^k\rightarrow u$ in $\text{BV}^\alpha(\Omega)$. We know that for any $\epsilon>0$ there exists a positive integer $N$ such that $\|u^k-u^j\|_{BV^\alpha(\Omega)}<\epsilon$ for any $j,k>N$, hence one has $\text{TV}^\alpha(u^k-u^j)<\epsilon$. 
Since $u^k\rightarrow u$ in $L^1(\Omega)$, thus $u^j-u^k\rightarrow u^j-u$ in $L^1(\Omega)$. Hence by Lemma \ref{lemma3-2},
\[\text{TV}^\alpha(u^j-u)\leq \liminf\limits_{k\rightarrow\infty}\text{TV}^\alpha(u^j-u^k)\leq\epsilon,\]
which shows that $u^k\rightarrow u$ in $\text{BV}^\alpha(\Omega)$, therefore $\text{BV}^\alpha(\Omega)$ is a Banach space.
\end{proof}

\begin{remark}\label{rem2}
In the literature \cite{IPodlubny1999}, the equivalence of different fractional derivatives requires stringent continuity conditions e.g.
 one has ${}^{C}D^\alpha_{[a,b]}\eta(x)=D^\alpha_{[a,b]}\eta(x)$ in
the test space $\mathscr{C}_0^{n-1}([a,b],\mathbb{R})$.
However for imaging applications (the objective function $u$), we do not require such equivalence.
\end{remark}

To distinguish the two definitions, we shall continue using the superscript $C$ for $C$ derivatives
based quantities such as ${}^C\diver^\alpha$ and ${}^C\nabla^\alpha$ while no superscript means that a quantity is
based on the R-L derivative.

For any positive integer $p\in\mathbf{N}^+$, let $W^\alpha_{p}(\Omega)=\big\{u\in L^p(\Omega)\;\big|\; \|u\|_{W^\alpha_{p}(\Omega)}<+\infty\big\}$ be a function space embedding with the norm
\[\|u\|_{W^\alpha_{p}(\Omega)}=\left(\int_\Omega |u|^pdx+\int_\Omega |\nabla^\alpha u|^pdx\right)^{1/p},\;\text{where}\;\nabla^\alpha u=(\frac{\partial^\alpha u}{\partial x_1},\dots,\frac{\partial^\alpha u}{\partial x_d})^T.\]

 For any $\xi(x)\in W^\alpha_1([a,b])$ and $\eta(x)\in\mathscr{C}_0^{n-1}([a,b],\mathbb{R})$
\begin{equation}\label{part_integral}
\begin{split}
&\int_a^b\xi(x)\cdot {}^{C}D^\alpha_{[a,b]}\eta(x)dx\\
=
&(-1)^n\int_a^b\eta(x)\cdot D^\alpha_{[a,b]}\xi(x)dx\  + \sum_{j=0}^{n-1}(-1)^jD^{\alpha-n+j}_{[a,b]}\xi(x)\frac{\partial^{n-j-1}\eta(x)}{\partial x^{n-j-1}}\Big|_{x=a}^{x=b}\\
=&(-1)^n\int_a^b\eta(x)\cdot D^\alpha_{[a,b]}\xi(x)dx
\end{split}
\end{equation}
 gives the $\alpha$-order integration by parts formula  (\emph{see} \cite{OPAgrawal2007}).

 Furthermore, applying (\ref{part_integral}) twice, we have shown the relationship
\begin{equation}\label{fractional_dual}
\int_\Omega u(x) \Big((-1)^n{}^C\diver^\alpha\Big)\phI(x)dx=\int_\Omega\phI(x)\cdot \nabla^\alpha u(x)dx,
\end{equation}
where $u(x)\in W_1^\alpha(\Omega)$ and $\phI(x)\in\mathscr{C}_0^\ell(\Omega,\mathbb{R}^d)$; clearly the operator $(-1)^n{}^C\diver^\alpha$ is the adjoint
of operator $\nabla^\alpha$. Note that, for $\phI(x)\in\mathscr{C}_0^\ell(\Omega,\mathbb{R}^d)$, we have ${}^C\diver^\alpha\phI = \diver^\alpha\phI$ which may not
be true if $\phI(x)$ is in a different space.

\begin{proposition}\label{proposition_regularization_term}
Assume that $u\in W_1^\alpha(\Omega)$, then
$\text{TV}^\alpha(u)=\int_\Omega|\nabla^\alpha u|dx$.
\end{proposition}
\begin{proof}
 For any $\alpha>0$, using the dual relationship (\ref{fractional_dual}), one can obtain that
 \begin{equation*}
 \begin{split}
 \int_\Omega u(x)\diver^\alpha\phI(x)dx\ =(-1)^n\int_\Omega\phI(x)\cdot \nabla^\alpha u(x)dx
 \end{split}
 \end{equation*}
and in addition $|\phI|\leq 1$ in $K$  implies that \[\phI_0(x)=\left\{\begin{array}{cc}
                                           (-1)^n\nabla^\alpha u/|\nabla^\alpha u|, & |\nabla^\alpha u(x)|\neq 0 ;\\
                                           0, & \text{otherwise}
                                         \end{array}\right.
\] can
 maximize the functional $\int_\Omega\phI(x)\cdot \nabla^\alpha u(x)dx= \int_\Omega|\nabla^\alpha u|dx$. By multiplying $\phI_0$ by a suitable characteristic $\ell$-compactly supported continuous function $\eta_\epsilon$ in $\Omega$ (e.g., $\eta_\epsilon\in\mathscr{C}_0^\ell(\Omega,\mathbb{R}^d)$) and then mollifying (\emph{see} \cite{RAcar1994} for TV and \cite{GAubert2006,giusti1984minimal}), the new $\int_\Omega\phI_\epsilon(x)\cdot \nabla^\alpha u(x)dx$ with $\phI_\epsilon\in K$ is arbitrarily close to $\int_\Omega|\nabla^\alpha u|dx$ as $\epsilon\rightarrow 0$ \cite{evans1991measure}, hence one shows that
 $\text{TV}^\alpha(u)=\int_\Omega|\nabla^\alpha u|dx$ by taking $\sup\limits_{\phI_\epsilon\in K}\int_\Omega u(x)\cdot \diver^\alpha\phI_\epsilon(x)dx\ =\sup\limits_{\phI_\epsilon\in K}(-1)^n\int_\Omega\phI_\epsilon(x)\cdot \nabla^\alpha u(x)dx$. 
\end{proof}


\begin{remark}\label{rem1}
Since $u\in W_1^\alpha(\Omega)$ leads to $\text{TV}^\alpha(u)=\int_\Omega|\nabla^\alpha u|dx$, in fact, it is easy to show that the lower semi-continuity
$\int_\Omega|\nabla^\alpha u|dx\leq \liminf\limits_{k\rightarrow\infty}\int_\Omega |\nabla^\alpha u^k|dx$ holds in the space $W_1^\alpha(\Omega)$ similar to the TV case \cite{evans1991measure}).
\end{remark}

\begin{lemma}
The space $W^\alpha_{p}(\Omega)$   is a Banach space.
\end{lemma}
\begin{proof} The $p=1$ case is clear. Now for $p\not=1$, let $q$ satisfy $1/p+1/q=1$.
To obtain the lower semi-continuity,
taking $u\in W_p^\alpha(\Omega)$ and $\psi(x)\in \big\{\phI\in \mathscr{C}_0^\ell(\Omega,\mathbf{R}^d)\;\Big|\;\|\phI(x)\|_{L^q(\Omega)}\leq 1 \;\text{for all}\;x\in\Omega\big\}$, the following inequality
\[\begin{split}
\int_\Omega(-1)^n\nabla^\alpha u\psi dx &=\   \int_\Omega u\diver^\alpha\psi dx
\   = \liminf\limits_{k\rightarrow +\infty}\int_\Omega u^k\diver^\alpha\psi dx\
=
\liminf\limits_{k\rightarrow +\infty}\int_\Omega(-1)^n\nabla^\alpha u^k\psi dx\\
&\leq \liminf\limits_{k\rightarrow +\infty}\left(\int_\Omega|\nabla^\alpha u^k|^pdx\right)^{1/p}\left(\int_\Omega|\psi|^qdx\right)^{1/q}
 \leq  \liminf\limits_{k\rightarrow +\infty}\left(\int_\Omega|\nabla^\alpha u^k|^pdx\right)^{1/p}
\end{split}\]
holds; further one has $\left(\int_\Omega|\nabla^\alpha u|^pdx\right)^{1/p}\leq \liminf\limits_{k\rightarrow +\infty}\left(\int_\Omega|\nabla^\alpha u^k|^pdx\right)^{1/p}$.
Then we can deduce the result, following the similar lines to proving Lemma \ref{lemma3-3}.
\end{proof}

\begin{lemma}
 The following embedding results hold:
\(  W^\alpha_{2}(\Omega)\subseteq W^\alpha_{1}(\Omega)\subseteq BV^\alpha(\Omega)\subseteq L^1(\Omega).\)
\end{lemma}
\begin{proof}
Firstly, from the definitions of $BV^\alpha(\Omega)$ and $W_p^\alpha(\Omega)$, we can see that  $BV^\alpha(\Omega)\subset L^1(\Omega)$ and $W_1^\alpha(\Omega)\subset L^1(\Omega)$.
Secondly, for any $f\in W_1^\alpha(\Omega)$ and $\phI\in K$, 
we have
\[\int_\Omega f\diver^\alpha\phI\;dx=(-1)^n\int_\Omega\phI(x)\cdot \nabla^\alpha f(x)dx\leq \|\nabla^\alpha f\|_{L^1(\Omega)}<+\infty,\]
i.e., $f\in BV^\alpha(\Omega)$ or $ W^\alpha_{1}(\Omega)\subseteq BV^\alpha(\Omega)$. Finally $W^\alpha_{2}(\Omega)\subseteq W^\alpha_{1}(\Omega)$ follows $L^2(\Omega)\subseteq L^1(\Omega)$.
\end{proof}

\begin{lemma}\label{lemma3.4}
The functional $\text{TV}^\alpha(u)$ is convex.
\end{lemma}
\begin{proof}
The proof follows the linearity of fractional order derivatives, and the positively homogeneous and sub-additive properties of $\text{TV}^\alpha(u)$.
\end{proof}

\label{section4}
{\bf Theory for a total $\alpha$-order variation model}.
We are now ready to analyze model (\ref{TV_alpha}) or
  the total $\alpha$-order
variation model in a more precise form
\begin{equation}\label{reg-problem1}
\min\limits_{u\in \text{BV}^\alpha(\Omega)} \Big\{E(u):=\text{TV}^\alpha(u)+\frac{\lambda}{2} F(u)\Big\},\qquad F(u)=\int_\Omega |u-z|^2 dx.
\end{equation}
To focus on the total $\alpha$-order variation model in $\Omega=(0,1)\times (0,1)\subset \mathbb{R}^2$,
  we assume $1< \alpha<2$;
the following theorem establishes convexity of the minimization problem (\ref{reg-problem1}).
\begin{theorem}[Convexity]\label{theorem_convexity}
The  functional $E(u)$ in $ \text{BV}^\alpha(\Omega)$ 
  is convex for $\lambda \geq 0$  and  strictly convex if $\lambda > 0$.
\end{theorem}
\begin{proof}
Since $F(u)$ is a strictly convex functional, the proof follows from  Lemma \ref{lemma3.4}.
\end{proof}

If a Banach space $X$ is reflexive (separable), then every bounded sequence in $X$ (in $X^*$) has a weakly ($weak^*$) convergent subsequence  \cite[Prop. 38.2]{EZeidler1985}.
Although $\text{BV}^\alpha(\Omega)$ is not reflexive, however, it is the dual of
a separable space. Therefore we can give the following definition:
\begin{definition}[A $weak^*$ topology]\label{definition3.14}
In $\text{BV}^\alpha(\Omega)$, a weak $\text{BV}^\alpha-w^*$ topology is defined as
\[u_j\xlongrightarrow[\text{BV}^\alpha-w^*]{*}u\;\;\Longleftrightarrow u_j\xlongrightarrow[L^1(\Omega)]{}u\;\text{ and }\; \int_\Omega \phI\cdot \nabla^\alpha u_j\;dx\xlongrightarrow\;\int_\Omega \phI \cdot\nabla^\alpha u\;dx\]
for all $\phI$ in $\mathscr{C}^0_0(\Omega,\mathbb{R}^d)$.
\end{definition}

From the above definition \ref{definition3.14}, we may derive the weak compactness of $\text{BV}^\alpha(\Omega)$ on the $weak^*$ topology.
This, combined with the weak lower semi-continuity of $E(u)$ and boundedness of Banach space $\text{BV}^\alpha(\Omega)$ (i.e, $u$ is  bounded in Banach space $\text{BV}^\alpha(\Omega)$), yields the following result:
\begin{theorem}[Existence]\label{theorem_existence0}
The  functional $E(u): BV^\alpha(\Omega)\rightarrow \mathbb{R}$ has a minimum.
\end{theorem}
\begin{proof} Follow the similar lines of \cite[Prop. 38.12(d)]{EZeidler1985}).
\qquad\hfill\end{proof}

\begin{theorem}[Uniqueness]\label{theorem3}
The functional $E(u)$ has a unique minimizer in $\text{BV}^\alpha(\Omega)$ when $\lambda > 0$.
\end{theorem}
\begin{proof} The convexity result of Theorem \ref{theorem_convexity}  leads to uniqueness of solutions. Refer to \cite[Theorem 47C]{EZeidler1985}.
\qquad\hfill\end{proof}

We remark that
similar  existence
and uniqueness theories of the total variation problem can be found
in \cite{RAcar1994,AChambolle1997,GAubert2006}.

\section{Nonzero Dirichlet boundary conditions and regularization}
\label{sec_bndy}
The standard definitions for fractional derivatives require a function to have  zero Dirichlet boundary conditions due to end singularity, but for imaging applications such conditions are unrealistic and too restrictive.
To obtain the system for finding the unknown   intensities $u$ at inner nodes of a discretization grids, we have to use  boundary conditions, but the difficulties caused by them in fractional derivative computations would be hard to overemphasize; inaccurate boundary conditions can easily lead to the oscillations near boundaries, so   proper treatment of the boundary conditions for problems involving fractional derivatives is crucial.

  In this section, we shall reduce nonzero Dirichlet boundary conditions to zero ones so that standard  definitions and our introduced algorithms become applicable.
The basic idea of boundary regularization is to introduce an auxiliary unknown function which satisfies the zero boundary conditions. In this way, the non-zero boundary conditions move  to the right-hand side of the equation as a new  known quantity.

We recall that, in the 1D case, if the boundary conditions
are nonzero
\[u(0)=a,\;u(1)=b,\]
we can reduce them to zero boundary conditions by introducing an auxiliary function
$e(x)=a(1-x)+bx$.
Precisely  taking
$\bar{u}(x)=u(x)-e(x)$ \cite{IPodlubny2000},
then
\[\bar{u}(0)=0,\;\bar{u}(1)=0;\; \bar{u}'(0)=\bar{u}'(1)=0 \]
and a Neumann boundary condition is imposed by artificially extending the boundary values i.e. $e'(0)=e'(1)=0$ on $\partial\Omega$.

Below we generalize the above 1D idea to the 2D case, assuming that the four corners of the solution are given or accurately estimated:
  \[u(0,0)=a,\;u(0,1)=b,\;u(1,0)=c,\;u(1,1)=d.\]
With $a,b,c,d$ known, at any image point $(x,y)\in \Omega$, a bilinear auxiliary function satisfying the above $4$ conditions
$e_1(x,y)=a+(c-a)x+(b-a)y+(d+a-c-b)xy$
can be   constructed to lead to
\begin{equation}\label{boundary_regularization1}
\bar{u}(x,y)=u(x,y)-e_1(x,y)
\end{equation}
which takes zero values at all $4$ corners.

If boundary conditions
$u(0,y)=a_1(y),\;u(1,y)=a_2(y),\;u(x,0)=b_1(x),\;u(x,1)=b_2(x)$
at $\partial\Omega$
are known a priori, then we can easily verify that
\begin{equation*}
\begin{split}
\bar{a}_1(y):&=\bar{u}(0,y)=a_1(y)-e_1(0,y),\;\;\bar{a}_2(y):=\bar{u}(1,y)=a_2(y)-e_1(1,y)
;\\
\bar{b}_1(x):&=\bar{u}(x,0)=b_1(x)-e_1(x,0),\;\;
\bar{b}_2(x):=\bar{u}(x,1)=b_2(x)-e_1(x,1)
\end{split}
\end{equation*}
define the new Dirichlet conditions for  $\bar{u}(x,y)$.

We can achieve zero conditions at the edges using the auxiliary function
$e_2(x,y)=\Big((1-x)\bar{a}_1(y)+x\bar{a}_2(y)\Big)
+\Big((1-y)\bar{b}_1(x)+y\bar{b}_2(x)\Big)$.
It is clear to see that the new image
$\tilde{u}(x,y)=u(x,y)-e_1(x,y)-e_2(x,y)$
satisfies
\begin{equation}\tilde{u}(x,y)|_{\partial\Omega}=0.
\end{equation}

\begin{remark}\label{remark4}
It remains to address the question of how to obtain estimates  of $u(x,y)$ at
corners and edges:
   \begin{enumerate}
    \item The true intensities $a:=u(0,0),\;b:=u(1,0);\;c:=u(1,0),\;d:=u(1,1)$ in four corner points are unknown a priori, to build the auxiliary function $e_1(x,y)$, the solutions approximating to them should be solved from the observed image $z(x,y)$ by the local smoothing or other simple techniques.
    \item Similarly, the true edge intensities $a_1(y)$, $a_2(y)$, $b_1(x)$ and $b_2(x)$ are also not given, a reconstruction step on boundary $\partial\Omega$ must be proceeded in order to capture a robust solution.
        To do this, we can apply a 1D model.
   \end{enumerate}
\end{remark}

According to Remark \ref{remark4},   we can propose a complete procedure  for regularizing boundary conditions for 2D variational image inverse problems in edges and corners.
\begin{itemize}
\item
Firstly, we restore image intensities in 4 corner points from an observed image $z$. A natural technique would be local smoothing operator for the region of corner points, the oscillations could also be reduced by many variational methods to local regions.
\item
Secondly,  in order to reconstructed 4 edges from the restored intensities $z(0,y)$, $z(1,y)$, $z(x,0)$ and $z(x,1)$, the total $\alpha$-order variation regularization would be used to solve four 1D inverse problems i.e. solve an equation like (\ref{reg-problem1}):
\begin{equation}\label{eq_1D}
   \min_{u} \{ E^{1D}(u) = \int_a^b | \frac{d^{\alpha}u}{dx^\alpha}| dx +
               \frac{\lambda^{1D}}{2}\int_a^b(u-z)^2 dx \}.
\end{equation}
\end{itemize}
 Thus through $e_1(x,y), e_2(x,y)$, we see that equation (\ref{boundary_regularization1}) reduces to finding the new image $\bar{u}(x,y)$ with zero Dirichlet conditions
 and hence the standard definitions of fractional derivatives for $\bar{u}(x,y)$ apply.

\section{Discretization and Split-Bregman algorithm}\label{sec_dis}
Since solution uniqueness of our variational model (\ref{reg-problem1}) is resolved, we now consider how to seek a numerical solution of
the total $\alpha$-order variation model.
We first  reformulate it in preparation for employment of
an efficient solver and then  discuss some discretization details (by finite-differences) before presenting our Algorithm 1.

\subsection{A Split-Bregman formulation}
Inspired by Goldstein and Osher's Split-Bregman work \cite{TGoldstein2009}, we introduce a special 
and new variable $\dd(x)=(d_1(x),d_2(x))^T$ to the total $\alpha$-order variation based model (\ref{reg-problem1}) to derive the following constrained
optimization problem:
\begin{equation}\label{constrainted_model}
\begin{split}
\min\limits_{u,\dd}&\int_\Omega |\dd|dx+\frac{\lambda}{2} F(u),\quad
\text{s.t.} \  \dd=\nabla^\alpha u.
\end{split}
\end{equation}
To enforce the constraint condition, we transfer it into the Bregman formulation
 \begin{equation*}
 \begin{split}
 (u^{k+1},\dd^{k+1})=&\min\limits_{u,\dd} \int_\Omega|\dd|dx+\frac{\lambda}{2} F(u)-\int_\Omega<\pp^k_d,\dd-\dd^k>dx\\
 &-\int_\Omega<\pp^k_u,u-u^k>dx
 +\frac{\mu}{2}
 \int_\Omega|\dd-\nabla^\alpha u|^2dx,\\
 \pp^{k+1}_u=&\pp^{k}_u-\mu(\nabla^\alpha )^T(\nabla^\alpha u^{k+1}-\dd^{k+1}),\\
  \pp^{k+1}_d=&\pp^{k}_d-\mu(\dd^{k+1}-\nabla^\alpha u^{k+1}).
 \end{split}
 \end{equation*}
The above iterative scheme can be simplified to the two-step algorithm \cite{TGoldstein2009,SSetzer2009,SSetzer2011}:
 \begin{equation}\label{sb_ud}
 \min\limits_{u,\dd} \int_\Omega|\dd|dx+\frac{\mu}{2}
 \int_\Omega|\dd-\nabla^\alpha u+
 \frac{\pp}{\mu}|^2dx+ \int_\Omega|\pp|^2dx + \frac{\lambda}{2} F(u)
 \end{equation}
 with the multiplier updated by iteration
 \(\pp^{k+1}=\pp^{k}-\gamma(\dd-\nabla^\alpha u),\)
 where $\pp(x)=(p_1(x),p_2(x))^T$.
Here the two main subproblems of (\ref{sb_ud}) are
  \begin{equation}\begin{split}
 \mbox{Subproblem } \dd:&\qquad
   \min\limits_{\dd} \int_\Omega|\dd|dx+\frac{\mu}{2}
 \int_\Omega|\dd-\nabla^\alpha u+  \frac{\pp}{\mu}|^2dx;\\
 \mbox{Subproblem } u: &\qquad
 \min\limits_{u} J(u):=
 \frac{\mu}{2}\int_\Omega|\dd-\nabla^\alpha u+\frac{\pp}{\mu}|^2dx+\frac{\lambda}{2} F(u).
 \label{eqn_u2}
 \end{split}
\end{equation}
Further note that
   the   subproblem $\dd$ has a closed-form solution \cite{TGoldstein2009}, while the  subproblem $u$   is determined by the associated Euler-Lagrange equation as shown below.
\begin{theorem}\label{reg-problem02}
Let $u(x)$ be a minimizer of functional $J(u)$ from (\ref{eqn_u2}). Then $u(x)$ satisfies the following first order optimal condition
\begin{equation}\label{fractionalSBEL}
(-1)^n\mu{}^{C}\diver^\alpha\left(\nabla^\alpha u-\dd-\frac{\pp}{\mu}\right)+\lambda (u-z)=0
\end{equation}
with one of these sets of boundary conditions
\[ \begin{split}
\mbox{i)}\ & \mbox{\rm fixed :} \
    u(x)\big|_{\partial\Omega}=b_1(x), \ \text{and }\
    \frac{\partial u(x)}{\partial n}\Big|_{\partial\Omega}=b_2(x);\\
\mbox{ii)}\ & \mbox{\rm homogeneous :} \
 D^{\alpha-2}
\left( \nabla^\alpha u-\dd-\frac{\pp}{\mu}\right)\cdot \mathbf{n}\Big|_{\partial\Omega} = 0, \ \
 D^{\alpha-1}
\left( \nabla^\alpha u-\dd-\frac{\pp}{\mu}\right)\cdot \mathbf{n}\Big|_{\partial\Omega} = 0
\end{split}\]
where $\mathbf{n}$ denotes the unit outward normal and $^C\diver^\alpha$ denotes the divergence operator based on the C derivative.
\end{theorem}
\begin{proof}
Refer to Appendix.
\end{proof}

\subsection{Discretization of the fractional derivative}
Before introducing the finite difference discretization of the fractional derivative, we define a spatial partition $(x_k,y_l)$ ( for all $k=0,1,\dots,N+1;l=0,1,\dots,M+1$) of image domain $\Omega$.
Assume $u$ has a zero Dirichlet boundary condition (practically we apply the
regularization method \S\ref{sec_bndy} first before discretization).
Here we mainly consider
the discretization of the $\alpha$-order fractional derivative
at the inner point $(x_k,y_l)$ (for all $k=,1,\dots,N;l=0,1,\dots,M$) on $\Omega$ along $x$-direction by using the approach
\begin{equation}\label{fractionaldiscretization}
\begin{split}
 D_{[a,b]}^\alpha f(x_k,y_l) &=\frac{\delta_0^\alpha f(x_k,y_l)}{h^\alpha}+O(h)\
=\frac{1}{2}\Big(\frac{\delta_-^\alpha f(x_k,y_l)}{h^\alpha}+\frac{\delta_+^\alpha f(x_k,y_l)}{h^\alpha}\Big)+O(h)\\
&=\frac{1}{2}\Big(h^{-\alpha}\sum_{j=0}^{k+1}
\omega_j^{\alpha}f^l_{k-j+1}+h^{-\alpha}\sum_{j=0}^{N-k+2}
\omega_j^{\alpha}f^l_{k+j-1}\Big)+O(h),
\end{split}
\end{equation}
which is applicable to both the R-L and C derivatives \cite{IPodlubny2009,HWang2014}, where $f^l_{s}=f_{s,l}$, $\omega_j^{(\alpha)}=(-1)^j\left(
                                   \begin{array}{c}
                                     \alpha \\
                                     j \\
                                   \end{array}
                                 \right)$, $j=0,1,\dots,N+1$ and
\[ \omega_0^{(\alpha)}=1; \omega_j^{(\alpha)}=(1-\frac{1+\alpha}{j})\omega_{j-1}^{(\alpha)}, \;
\mbox{for}\; j>0.\]
Alterative discretization for fractional derivatives in the Fourier space can be found in \cite{JBai2007,PGuidotti2009b}.

Observe from (\ref{fractionaldiscretization}) that the first order estimate of the $\alpha$-order fractional $D_{[a,b]}^\alpha f(x_k,y_l)$ along $x$-direction at the point $(x_k,y_l)$ with a fixed $y_l$ is a linear combination of $N+2$ values $\{f^l_{0},f^l_1,\dots,f^l_{N},f^l_{N+1}\}$.

After incorporating  zero boundary condition  in the matrix approximation of fractional derivative,
  all $N $ equations of fractional derivatives along $x$ direction in 
  (\ref{fractionaldiscretization}) can be written simultaneously in the matrix form
  (denote $w=\omega_0^{\alpha}+\omega_2^{\alpha}$):
\begin{equation}\label{eq5.1}
\begin{split}
 \left(
  \begin{array}{c}
    \delta_0^\alpha f(x_1,y_l) \\
    \delta_0^\alpha f(x_2,y_l) \\
   \vdots \\
   \vdots \\
    \delta_0^\alpha f(x_{N},y_l) \\
  \end{array}
\right)\ = \underbrace{\frac{1}{2h^\alpha}\left(
                            \begin{array}{ccccc}
                               2\omega_1^{\alpha} & w
                               & \omega_3^{\alpha} &  \cdots &  \omega_N^{\alpha} \\
                              w & 2\omega_1^{\alpha} & \ddots &  \ddots&  \vdots \\
                              \omega_3^{\alpha} &  \ddots &\ddots & \ddots & \omega_3^{\alpha} \\
                              \vdots & \ddots & \ddots  & 2\omega_1^{\alpha} & w\\
                              \omega_N^{\alpha} & \cdots & \omega_3^{\alpha} & w & 2\omega_1^{\alpha} \\
                            \end{array}
                          \right)}_{B^\alpha_N}\underbrace{\left(
                                   \begin{array}{c}
                                     f^l_1 \\
                                     f^l_2 \\
                                     \vdots \\
                                    \vdots \\
                                     f^l_{N} \\
                                   \end{array}
                                 \right)}_{f}.
                                 \end{split}
\end{equation}
From the definition of fractional order derivative (\ref{fractionaldiscretization}), for any $1< \alpha<2$, the coefficients $\omega_k^{(\alpha)}$ has the following properties \cite{IPodlubny1999,HWang2014}:
\begin{description}\itemindent=1cm
  \item[1).]  $\omega_0^{(\alpha)}=1,\; \omega_1^{(\alpha)}=-\alpha<0$,\;\qquad\textbf{2).}\;$1\geq \omega_2^{(\alpha)}\geq \omega_3^{(\alpha)}\geq\dots\geq 0$,
  \item[3).] $\sum_{k=0}^{\infty}\omega_k^{(\alpha)}=0$,\; \quad\quad\quad\qquad\quad \textbf{4).}\; $\sum_{k=0}^{m}\omega_k^{(\alpha)}\leq 0 \;(m\geq 1)$.
\end{description}
Hence by the Gerschgorin circle theorem, one can derive that
matrix $B^\alpha_N$ in (\ref{eq5.1}) is a symmetric and negative definite Toeplitz matrix (i.e. $-B^\alpha_N$ is a positive definite Toeplitz matrix).

We recall that the Kronecker product $A\otimes B$ of the $p\times q$ matrix $A=[a_{ij}]$ and the $n\times m$ matrix $B=[b_{rt}]$
is the $np\times mq$ matrix having the block structure $A\otimes B:=[a_{ij}B]$. Further vector $(A\otimes B)x$ can be computed by matrix scheme
$BXA^T$ (i.e., $[(A\otimes B)x]_s=[BXA^T]_{j,i}$ with $s=(i-1)m+j$),
where the $m\times q$ matrix $X$ is the reshape of the vector $x$ along its column.

Let $U\in {\mathbb R}^{N\times M}$ denote the solution matrix at all nodes $(kh_x; lh_y )$, $k=1,\dots,N;\;l=1,\dots,M$, corresponding to
 x-direction and y-direction spatial discretization nodes.
Denote by $\vec{u}\in {\mathbb R}^{N\!\!M\times 1}$ the ordered solution vector of $U$.
The direct and  discrete analogue of differentiation of arbitrary $\alpha$
order derivative is
\[u^{(\alpha)}_x=(I_M\otimes B^\alpha_N)\vec{u}=B_x^{(\alpha)}\vec{u},\]
where
\(\displaystyle {u}^{(\alpha)}_x=\left(u^{(\alpha)}_{11},\dots,u^{(\alpha)}_{N1},
u^{(\alpha)}_{12},\dots,u^{(\alpha)}_{NM}\right)^T,\;\vec{u}
=\left(u_{11},\dots,u_{N1},u_{12},\dots,u_{NM}\right)^T.\)
Similarly, the $\alpha$-th order y-direction derivative of $u(x; y)$
is approximated by:
\[u^{(\alpha)}_y=B_y^{(\alpha)}\vec{u}=(B^\alpha_M\otimes I_N)\vec{u},\;\;\text{where}\;{u}^{(\alpha)}_y=\left(u^{(\alpha)}_{11},\dots,u^{(\alpha)}_{1M},u^{(\alpha)}_{21},\dots,u^{(\alpha)}_{NM}\right)^T.\]

\subsection{The Split-Bregman algorithm}
In discrete form, we are ready to state the discretized equations in structured matrix form. The discrete scheme of 
(\ref{fractionalSBEL}) is  given by
\[\begin{split}
(-1)^n\mu\Bigg(\Big((B_x^{(\alpha)})^T(B_x^{(\alpha)}\vec{u})+&(B_y^{(\alpha)})^T(B_y^{(\alpha)}\vec{u})\Big)-\Big(B_x^{(\alpha)})^T \vec{d}_1+B_y^{(\alpha)})^T\vec{d}_2\Big)\\-&\frac{1}{\mu}\Big(B_x^{(\alpha)})^T \vec{p}_1+B_y^{(\alpha)})^T\vec{p}_2\Big)\Bigg)+\lambda(\vec{u}-\vec{z})=0
\end{split}\]
with discretizations $\vec{d}_i=\left(d^i_{11},\dots,d^i_{N1},d^i_{12},\dots,d^i_{NM}\right)^T$ and $\vec{p}_i=\left(p^i_{11},\dots,p^i_{N1},p^i_{12},\dots,p^i_{NM}\right)^T$ of vectors $\dd$ and $\pp$ ($i=1,2$). A matrix approximation equation
is given as
\begin{equation}\label{sbeq}
\begin{split}
\underbrace{\Big((B^\alpha_N)^T(B^\alpha_NU)+U(B^\alpha_M)^TB^\alpha_M\Big)+\bar{\lambda} U}_{WU}=\underbrace{\bar{\lambda} Z+\Big((B^\alpha_N)^T D_1+D_2B^\alpha_M\Big)+\frac{1}{\mu}\Big((B^\alpha_N)^T P_1+P_2B^\alpha_M\Big)}_{F},
\end{split}\end{equation}
where $D_i$ and $P_i$ are $N\times M$-size reshape matrices of vectors $\vec{d}_i$ and $\vec{p}_i$ for $i=1,2$, $\bar{\lambda}=(-1)^n\lambda/\mu$.
The following   justifies the use of a conjugate gradient method for $WU=F$.

\begin{theorem}\label{theorem5.5}\mbox{}
The weighted matrices inner product $\langle WU,U\rangle=\sum\limits_{ij}(\sum\limits_{k} W_{ik}U_{kj})U_{ij}$ is positive for any matrix $U\neq 0$, where $W$ is a known positive definite operator.
\end{theorem}
\begin{proof}
For any matrix $U\neq 0$, it is easy to show that
\[\begin{split}
\langle WU,U\rangle&=\Big\langle \big((B^\alpha_N)^T(B^\alpha_NU)+U(B^\alpha_M)^TB^\alpha_M\big)+\bar{\lambda} U,\  U\Big\rangle\\
&=\langle B^\alpha_NU, B^\alpha_NU\rangle+\langle U(B^\alpha_M)^T, U(B^\alpha_M)^T\rangle+\bar{\lambda} \langle U, U\rangle\\
&=\| B^\alpha_NU\|_F^2+\| U(B^\alpha_M)^T\|_F^2+\bar{\lambda}\|U\|_F^2>0,
\end{split}\]
which completes the proof.
\end{proof}

An implementation of this method may be summarized below:
\begin{alg}[Split-Bregman iterations (PDE-SB)]\label{algsb02}\mbox{}
\begin{enumerate}[step 1.]
\item Boundary regularization 
for an observed image $z$;\label{alg02-00}
\item Given initial matrices $P_1^{k=0}$, $P_2^{k=0}$ and $U^{k=0}$;\label{alg02-01}
\item Solve subproblem $\mathbf{d}$: Compute the auxiliary matrix $\Bigg(
                                           \begin{array}{c}
                                             D_1 \\
                                             D_2 \\
                                           \end{array}
                                         \Bigg)
$  from the closed form solution \[\Bigg(\begin{array}{c}D_1 \\D_2 \\\end{array}\Bigg)^{k+1}=shrink\left(\Bigg(
                                           \begin{array}{c}B^\alpha_NU^{k+1} \\U^{k+1}(B^\alpha_M)^T \\\end{array}
                                         \Bigg)+\Bigg(
                                           \begin{array}{c}
                                             P_1 \\
                                             P_2 \\
                                           \end{array}
                                         \Bigg)^k,{1\over\mu}\right) \]\label{alg02-02}
    by solving the Moreau-Yosida problem with the $l^1$ regularization;
\item Solve subproblem $u$: Find the solution $U^{k+1}$ of (\ref{sbeq}) with an effective parameter $\lambda$ $\mu$ by CG method;\label{alg02-03}
\item Update $\Bigg(\begin{array}{c} P_1 \\P_2 \\\end{array}\Bigg)^{k+1}=\Bigg(\begin{array}{c} P_1 \\P_2 \\ \end{array}\Bigg)^{k}+\gamma\left(\Bigg(\begin{array}{c}B^\alpha_NU^{k+1} \\U^{k+1}(B^\alpha_M)^T \\\end{array}
                                         \Bigg)-\Bigg(\begin{array}{c}D_1 \\D_2 \\\end{array}\Bigg)^{k+1}\right)$ with $\gamma\in (0,1]$;\label{alg02-04}
\item Check the stopping condition; \label{alg02-05}
\begin{itemize}
\item If $|U^k-U^{k+1}|<\epsilon$,\\ \text{ }
    stop and return $U^*:=U^{k+1}$;
\item else\\
\text{ \ }   $k:=k+1$, go back to Step \ref{alg02-02};\label{alg02-06}
\item  end
\end{itemize}
\item Accept the correct solution $U$ from boundary regularization.
\end{enumerate}
\end{alg}

\section{Optimization based numerical methods}
As many variational models are increasingly solved by the discretise-optimise approach,
we now present three related algorithms for model (\ref{reg-problem02})
 after applying a finite difference discretization.
In this section, we assume that we have the   zero Dirichlet boundary conditions for $u$
mainly to simplify the notation.

As   in \S \ref{sec_dis}, the $\alpha$-th order derivative $u_x^{(\alpha)}$ of $u(x; y)$ along  all
$x$-direction  nodes in $\Omega$ can be given by matrix $B^\alpha_NU$, and similarly $U(B^\alpha_M)^T$ for $y$-direction (as $U$ is the solution matrix).

Define $\langle U,V\rangle=\sum\limits_{ij}U_{ij}V_{ij}$ and let $V_1=\{p\ |\ 0 \leq p\leq1\},\  V_2=\{p\ |\ |p|\leq1\}$. Then
using the discrete setting introduced above, the discretised  problem of model (\ref{reg-problem1}) is
\begin{equation}\label{discrete_eq1}
\min\limits_{U\in V_1}\max\limits_{\Phi\in V_2} G(U,D^*\Phi)+\frac{\lambda}{2}H(U)
\end{equation}
where $H(U)=\sum\limits_{ij}(U_{ij}-Z_{ij})^2$ and $G(U,D^*\Phi)=\langle U,D^*\Phi\rangle=\sum\limits_{ij} U_{ij}\left(B^\alpha_N\Phi_1+\Phi_2(B^\alpha_M)^T\right)_{ij}$, due to $D^*\Phi=B^\alpha_N\Phi_1+\Phi_2(B^\alpha_M)^T$.
We also have the adjoint relationship $\langle U,D^*\Phi\rangle= \langle DU,\Phi\rangle$
with $DU = (B^\alpha_NU,U(B^\alpha_M
)^T )$ and $\Phi= (\Phi_1,\Phi_2)$.
In line with the literature, this model can be denoted by the convex optimization problem in a generic notation by
\begin{equation}\label{min-f-g}
\min\limits_{x} \left\{f_1(x)+f_2(x)\right\}
\quad \mbox{i.e.} \
\min\limits_{x, y} \left\{f_1(x)+f_2(y)\right\}\  \mbox{s.t.} \ x=y
\end{equation}
where one views $x=U$, $f_1(x)=\max_{\Phi\in V_2}G(U,D^*\Phi), \ f_2(x)=H(U)$.
We also need the notation
\[\text{prox}^\lambda_{ f_1}(x^k):=\argmin_{x\in V_1}\left\{f_1(x) + \frac{1}{2\lambda}\|x-x^k\|^2\right\}\]
where $f_1$ can be any other convex function and $\lambda>0$.

To solve (\ref{min-f-g}) by the methods to be presented, computation of the proximal point $\text{prox}^\lambda_{f_1}(x)$ is a major and nontrivial step.
We consider how to compute it when  $D=\nabla^\alpha$, borrowing ideas from solving
a similar problem of TV regularization.
In a dual setting, Chambolle \cite{AChambolle2004,AChambolle2011} firstly proposed a discrete dual method by optimizing a cost function consisting of two variants
 \cite{AChambolle2004,DLChen2013a}. Recently, one variant of this scheme is employed in \cite{DLChen2013a} to effectively solve a fractional image model by a dual transform. The other variant
is used in  \cite{DLChen2013c}.

Define two projections as
\[\text{Proj}_{V_1}(p)=\left\{\begin{array}{ll}
                  0 & p<0  \\
                  p & 0\leq p\leq 1\\
                  1 &  1\leq p,
                \end{array}\right.\qquad
        \text{Proj}_{V_2}(p)  =\frac{p}{\max(1, \|p\|)}.
                \]
Noting $\frac{\partial G(x,D^*\Phi)}{\partial x}=D^*\Phi$ and that the optimal solution
  is
\begin{equation}\label{discrete_solution}
 x=\text{prox}^\gamma_{f_1}(x^k)=\text{Proj}_{V_1}(\bar{x}),
\end{equation}
where $\bar{x}=x^k-\gamma D^*\Phi$ and $\Phi$ is unknown. Based on methods of Chambolle \cite{AChambolle2004} and
      Beck-Teboulle \cite{ABeck2009b}, we see that
      (\ref{discrete_solution})
            can be used to
      reduce the min-max problem
      \[\min_{x\in V_1}\left\{\max_{\Phi\in V_2}\langle x,D^*\Phi\rangle +
\frac{1}{2\gamma}\|x-x^k\|^2\right\}\] to the dual problem $\max_{\Phi\in
V_2}\langle \text{Proj}_{V_1}(\bar{x}),D^*\Phi\rangle +
\frac{1}{2\gamma}\|\text{Proj}_{V_1}(\bar{x})-x^k\|^2$
and further  to
\[
\begin{split}
&\langle \text{Proj}_{V_1}(\bar{x}),D^*\Phi\rangle +
\frac{1}{2\gamma}\|\text{Proj}_{V_1}(\bar{x})-x^k\|^2\quad
=\langle \text{Proj}_{V_1}(\bar{x}),D^*\Phi\rangle\\
 & +
\frac{1}{2\gamma}\|\text{Proj}_{V_1}(\bar{x})-x^k+\gamma D^*\Phi\|^2-\frac{1}{2\gamma}\|\gamma D^*\Phi\|^2-\frac{1}{2\gamma}2\langle \text{Proj}_{V_1}(\bar{x})-x^k,\gamma D^*\Phi\rangle\\
=&
\frac{1}{2\gamma}\|\text{Proj}_{V_1}(\bar{x})-(x^k-\gamma D^*\Phi)\|^2-\frac{1}{2\gamma}\|\gamma D^*\Phi\|^2+\frac{1}{2\gamma}2\langle x^k,\gamma D^*\Phi\rangle
\\
=&
\frac{1}{2\gamma}\|\text{Proj}_{V_1}(\bar{x})-(x^k-\gamma D^*\Phi)\|^2-\frac{1}{2\gamma}\left(\|\gamma D^*\Phi\|^2-2\langle x^k,\gamma D^*\Phi\rangle+\|x^k\|^2\right)+\frac{1}{2\gamma}\|x^k\|^2
\end{split}
\]
\begin{equation}\label{eq3}
\begin{split}
=&
\frac{1}{2\gamma}\|\text{Proj}_{V_1}(\bar{x})-(x^k-\gamma D^*\Phi)\|^2-\frac{1}{2\gamma}\|x^k-\gamma D^*\Phi\|^2+\frac{1}{2\gamma}\|x^k\|^2\\
=&\frac{1}{2\gamma}\left(\|\bar{x}-\text{Proj}_{V_1}(\bar{x})\|^2-\|\bar{x}\|^2+\|x^k\|^2\right),
\end{split}
\end{equation}
i.e. $\max_{\Phi\in V_2}\langle \text{Proj}_{V_1}(\bar{x}),D^*\Phi\rangle +
\frac{1}{2\gamma}\|\text{Proj}_{V_1}(\bar{x})-x^k\|^2=-\frac{1}{2\gamma}\min\limits_{\Phi\in
V_2} h(\Phi)$ where $h(\Phi)=\|x^k-\gamma D^*\Phi\|^2-\|(x^k-\gamma
D^*\Phi)-\text{Proj}_{V_1} (x^k-\gamma
D^*\Phi)\|^2-\|x^k\|^2=\|\bar{x}\|^2-\|\bar{x}-\text{Proj}_{V_1}(\bar{x})\|^2-\|x^k\|^2$.

Below we consider the operator
$S(\bar{x})=\|\bar{x}-\text{Proj}_{V_1}(\bar{x})\|^2=\inf\limits_y
\left\{\delta_{V_1}(y)+\frac{1}{2\gamma}\|y-\bar{x}\|^2\right\}$. Since its
gradient is $\nabla_{\bar{x}} S(\bar{x})=2(\bar{x}-Proj_{V_1}(\bar{x}))$,
  we get
\[\nabla_\Phi h(\Phi)=-2\gamma D(\text{Proj}_{V_1} (x^k-\gamma D^*\Phi)).\]
The minimization problem
$\min\limits_{\Phi\in V_2} h(\Phi)$
can be solved to obtain the $\Phi$-update as follows
\begin{enumerate}[1)]
  \item \(\bar{\Phi}=\Phi^n-L(h)\nabla_\Phi h(\Phi^n); \)
  \item
      \(\Phi^{n+1}=\text{Proj}_{V_2}(\bar{\Phi})=\text{Proj}_{V_2}\left(\Phi^n+2L(h)\gamma
      D(\text{Proj}_{V_1} (x^k-\gamma D^*\Phi^n))\right)\)
\end{enumerate}
using
the gradient projection scheme of $h(\Phi)$    \cite{ABeck2009b}.
Here $L(h)\leq 16\gamma^2$ is the Lipschitz constant. Finally   the proximal point $\text{prox}_{f_1}^{\gamma}(x^k)$ is given by (\ref{discrete_solution}) once $\Phi$ is obtained; see also \cite{ABeck2009b}.

\subsection{Forward-backward algorithm}
Various applications in sparse optimizations stimulated the search for simple and efficient
first-order methods.
The forward backward scheme for  (\ref{min-f-g}) is
based (as the name suggests) on  recursive application
of an explicit forward step with respect to $f_2$, i.e,
\[\min\limits_{x}\Big\{
\underbrace{f_2(x^k)+\langle \nabla f_2(x^k), x-x^k\rangle}_{l(x)}+\frac{1}{2\gamma}\|x-x^k\|^2\Big\},\]
and followed by
an implicit backward step with respect to $f_1$, i.e.,
\begin{equation}\label{subprob-g}
\min\limits_{x}\Big\{f_1(x)+\frac{1}{2\gamma}\|x-x^k\|^2\Big\}.
\end{equation}
The scheme decouples the contributions of the functions $f_1$ and $f_2$ in a gradient
descent step 
\cite{bauschke2011fixed}.
The scheme is also known under the name of proximal gradient methods \cite{ZWShen2011,combettes2005signal,SSetzer2009,bauschke2011fixed},
since the implicit step relies on the computation of the so-called proximity operator.

 The forward backward
algorithm is summarised as follows.
\begin{alg}[Forward-backward algorithm (FB)\cite{bauschke2011fixed}]\label{algfb}\mbox{}
\begin{itemize}
\item Fix initial $x_0$, set $\epsilon\in [0, \min\{1,1/\beta\}]$, $\beta$ (a Lipschitz parameter);
\item For $k\geq 0$
\begin{enumerate}[ Step 1.]\itemindent=1cm
\item  $\gamma_k\in[\epsilon,2/\beta-\epsilon]$, $\lambda_k\in [\epsilon,1] $;
\item $y_k=\text{prox}^{\gamma_k}_{\l}(x_k)$ 
 \item $ x_{k+1}=\text{prox}^{\gamma_k}_{f_1}(y_k)$;
 \item  $x_{k+1}=x_k+\lambda_k(x_{k+1}-x_k)$;
 \item Stop when $\|x_{k+1}-x_k\|$ is small enough otherwise continue.
\end{enumerate}

\end{itemize}
\end{alg}

\subsection{Nesterov's method} 
 As a gradient based
method, though simple, the above method can exhibit a slow speed of convergence. For this reason,
Nesterov \cite{YNesterov1983} proposed an improved
gradient method
aiming to accelerate and modify   the classical forward-backward splitting
algorithm, while achieving  an almost optimal convergence rate.
As a consequence of this
breakthrough, a few recent works have followed up the idea and improved techniques for some
  specific problems in signal or image processing \cite{ABeck2009b,JFAujol2009}.

Recently Nesterov \cite{YNesterov2013} presented an accelerated multistep version, which converges as $O(\frac{1}{r^2})$ ($r$ is the iteration number). For a problem of type (\ref{min-f-g}), this new method introduced a composite gradient mapping.
We now show the algorithm as follows.
\begin{alg}[Nesterov accelerated method (Nesterov \cite{YNesterov2013})]\label{alg_nesterov}\mbox{}
\begin{itemize}
\item Fix initial $x_0$, $b_0$, set $y_0=x_0$ and $\beta$ (a Lipschitz parameter);
\item For $k\geq 0$
\begin{enumerate}[ Step 1.]\itemindent=1cm
\item  Find $a=a_k$ from the quadratic equation
           $\frac{a^2}{2(b_k +a)} =\frac{1+b_k }{\beta}$;
\item $v_ =\text{prox}^{b_k}_{f_1}(x_k-y_k)$;
\item $z_{k+1}=\frac{b_kx_k+a_kv_k}{b_k+a_k}$;
\item $x_{k+1}=\text{prox}^{\beta^{-1}}_{f_1}(z_{k+1}-\beta^{-1}\nabla f_2(z_{k+1}))$;
 \item $y_{k+1}=y_{k}+a_{k}\nabla f_2(x_{k+1})$;
 \item  $b_{k+1}=b_k+a_k$;
 \item Stop when $\|x_{k+1}-x_k\|$ is small enough otherwise continue.
\end{enumerate}
\end{itemize}
\end{alg}

\subsection{FISTA method}
Beck and Teboulle \cite{ABeck2009b,ABeck2009a} proposed a fast iterative shrinkage thresholding algorithm (FISTA)   to solve the image denoising and deblurring model, The method applies the idea  of Nesterov to the forward-backward splitting framework,
resulting in the same optimal
convergence rate as Nesterov’s method but wider applicability.
It can be applied to a variety of practical
problems arising from sparse signal recovery, image processing
and machine learning and hence has become a standard
algorithm.

Applying it to (\ref{min-f-g}), we obtain Algorithm \ref{algfgp} below.
\begin{alg}[FISTA (Beck-Teboulle \cite{bauschke2011fixed,ABeck2009b,ABeck2009a})]\label{algfgp}\mbox{}
\begin{itemize}
\item Fix initial $x_0$, set $z_0=x_0$ and $t_0=1$, $\beta$ (a Lipschitz parameter);
\item For $k\geq 0$
\begin{enumerate}[ Step 1.]\itemindent=1cm
\item  $y_k=z_k-\beta^{-1}\nabla f_2(z_k)$
\item $x_{k+1}=\text{prox}^{\beta^{-1}}_{f_1}(y_k)$
 \item $t_{k+1}=\frac{1+\sqrt{4t_k^2+1}}{2}$
 \item  $z_{k+1}=x_k+(1+\frac{t_k-1}{t_k})(x_{k+1}-x_k)$;
 \item Stop when $\|x_{k+1}-x_k\|$ is small enough otherwise continue.
\end{enumerate}
\end{itemize}
\end{alg}

\section{Numerical results}\label{sec_num}
 Finally, we present some numerical results from using the four presented algorithms denoted by
 \begin{description}\itemindent=-1mm
   \item[PDE-SB:] \qquad  \quad PDE-based Split-Bregman (Algorithm \ref{algsb02});
   \item[Opti-FB:] \qquad  \quad Optimization based Forward-backward (Algorithm \ref{algfb});
   \item[Opti-Nesterov:] Optimization based Nesterov Accelerated method (Algorithm \ref{alg_nesterov});
   \item[Opti-FISTA:] \ \quad Optimization based FISTA (Algorithm \ref{algfgp}),
 \end{description}
 and their comparisons with related methods.
 In all tests, an initial solution is
the noisy image $z(x,y)$, the algorithms solving the
diffusion equation or optimization problem are stopped after achieving a
relative residual of $10^{-4}$ or a relative error of $10^{-8}$ within 1000
outer and 15 inner iterations.
Here we mainly compare the solution's visual quality, the \emph{snr} (the signal-to-noise ratio) and \emph{psnr} (the peak signal-to-noise ratio) values which are given
 \[\emph{snr}(u,u^*)=10\log_{10}\frac{\|u^*-mean(u^*)\|_F^2}{\|u-u^*\|_F^2};\;\emph{psnr}(u,u^*)=10\log_{10}\frac{n_xn_y\Big(\max\limits_{i,j} u_{i,j}^*\Big)^2}{\|u-u^*\|_F^2},\]
 where $mean(u^*)$ is an average value of the true image $u^*$, $n_x$ and $n_y$ denote the size of the test image $z$. It should be noted however that these valuations do not always correlate with human perception. In real life situations, the two measures are also not possible because the true image is not known.

 In general, an optimization problem may be solved many times to select
  a suitable regularization parameter  $\lambda$ or to optimize the solution for the underlying inverse problem; a solution is accepted when some stopping criterion is satisfied. It remains to carry out a systematic study
 on our new model as in  \cite{JPZhang2012} for the TV model.
 However we shall use the best (numerical) $\lambda$  for all models in the following tests.

For denoising,   $F(u)=(u-z)^2$  is the $L^2$ measure between the solution $u$ and the observed image $z$. To intuitively describe the denoising ability, four sets of data will be used in this part (\emph{also see} Fig.\ref{figure-example}):
\begin{description}\itemindent=1mm
  \item[P1:] \emph{Problem 1} - Parabolic surfaces;
  \hspace*{2cm}${\bf P2:}$  \emph{Problem 2} - Saddle surface;
  \item[P3:] \emph{Problem 3} - Pepper;
  \hspace*{3.84cm}${\bf P4:}$ \emph{Problem 4} - Penguin.
\end{description}
\begin{figure}[!h]
\begin{center}
\includegraphics[width=5.6in,height=0.9in]{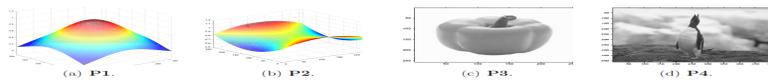}\label{figexample}
\end{center}
\caption{Test datasets.}\label{figure-example}
\end{figure}
 Though our framework is readily applicable to image deblurring and image registration, here we only
present denoising results.
\subsection{Performance comparisons of boundary regularization} 
We first test the idea from \S\ref{sec_bndy}.
One the hand, the variational framework seeks the boundary conditions of a nonzero
    Dirichlet or a Neumann type on $\partial\Omega$ and also real images do have nonzero boundary conditions.
On the other hand, fractional order derivatives require homogeneous boundary conditions (as used in works of many authors) due to end singularity. In order to aid accurate computation of the discretized fractional order derivative,   in our work, a boundary processing technique \S\ref{sec_bndy} has been proposed to transform nonzero boundary conditions of observed data $z$ into zero boundary conditions; hence a consequent matrix approximation to the fractional derivative operator $D^\alpha_{[0,\;1]}$ can use a zero Dirichlet boundary condition.

Here we test the performance and effectiveness of our boundary regularization against no regularization. The experiment is carried out on \textbf{P1} - Parabolic surfaces as shown in Fig. \ref{figure-1}, i.e., a synthetic image of size $256\times 256$ and range [0, 1], in Fig. \ref{fig1}(a), which is added zero mean value Gaussian random noise with a mean variance $\delta=\frac{15}{256}$ to get
 the noisy image   displayed in
Fig. \ref{fig1}(d). For the boundary regularization case, the approximation $u|_{\partial\Omega}$ from the observed data $z|_{\partial\Omega}$ is from applying one dimensional fractional order variation model as described in \S \ref{sec_bndy}. The treated case is
named as `Treated' whose results are depicted on Fig. \ref{fig1}(b) and Fig. \ref{fig1-e}), where \emph{psnr}$=47.53$ and \emph{snr=}$35.43$. The solution obtained from assuming zero boundary
conditions for $u$ is named as `Non-treated' with its results depicted in Fig. \ref{fig1}(c) and Fig. \ref{fig1-f}, where $\emph{psnr}=23.69$ and $\emph{snr}=10.38$.
Clearly our boundary regularization treatment is effective.

\begin{figure}[!h]
\begin{center}
\includegraphics[width=5.6in,height=3.1in]{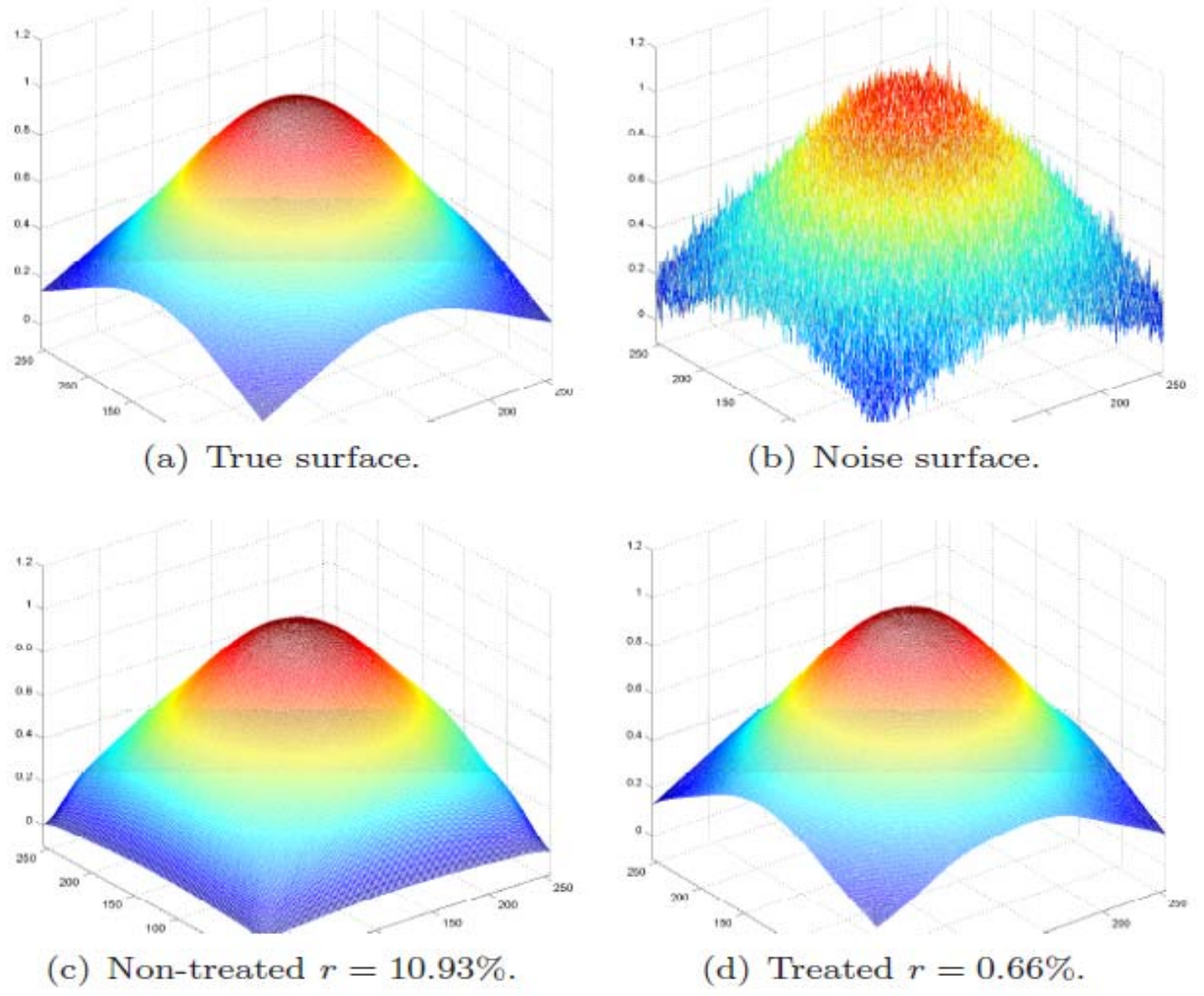}\label{fig1}\\
\subfigure[Slice for non-treated - Bad.]{\label{fig1-f}
\includegraphics[width=2.0in,height=1.80in]{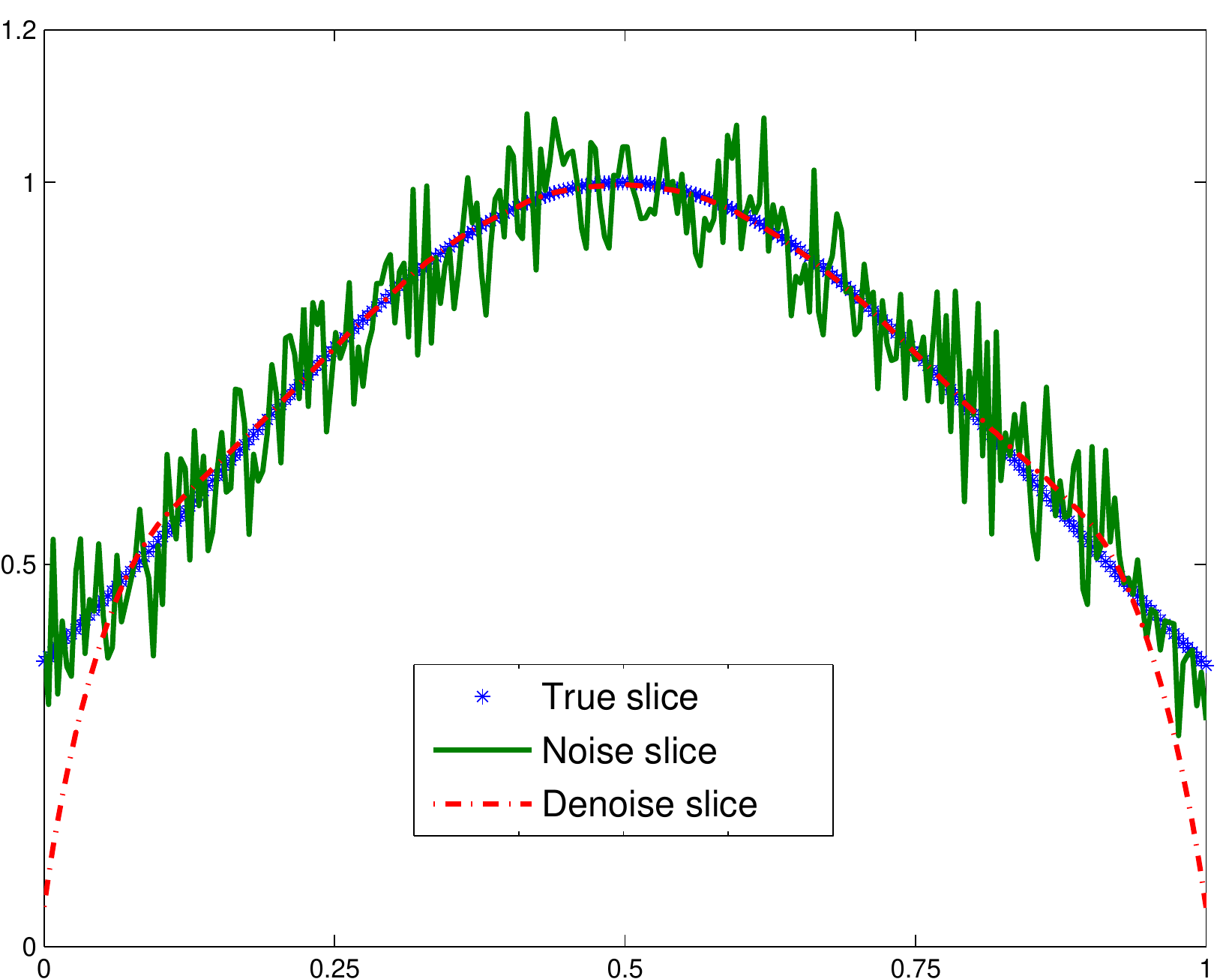}}
\subfigure[Slice for treated - Good.]{\label{fig1-e}
\includegraphics[width=2.0in,height=1.80in]{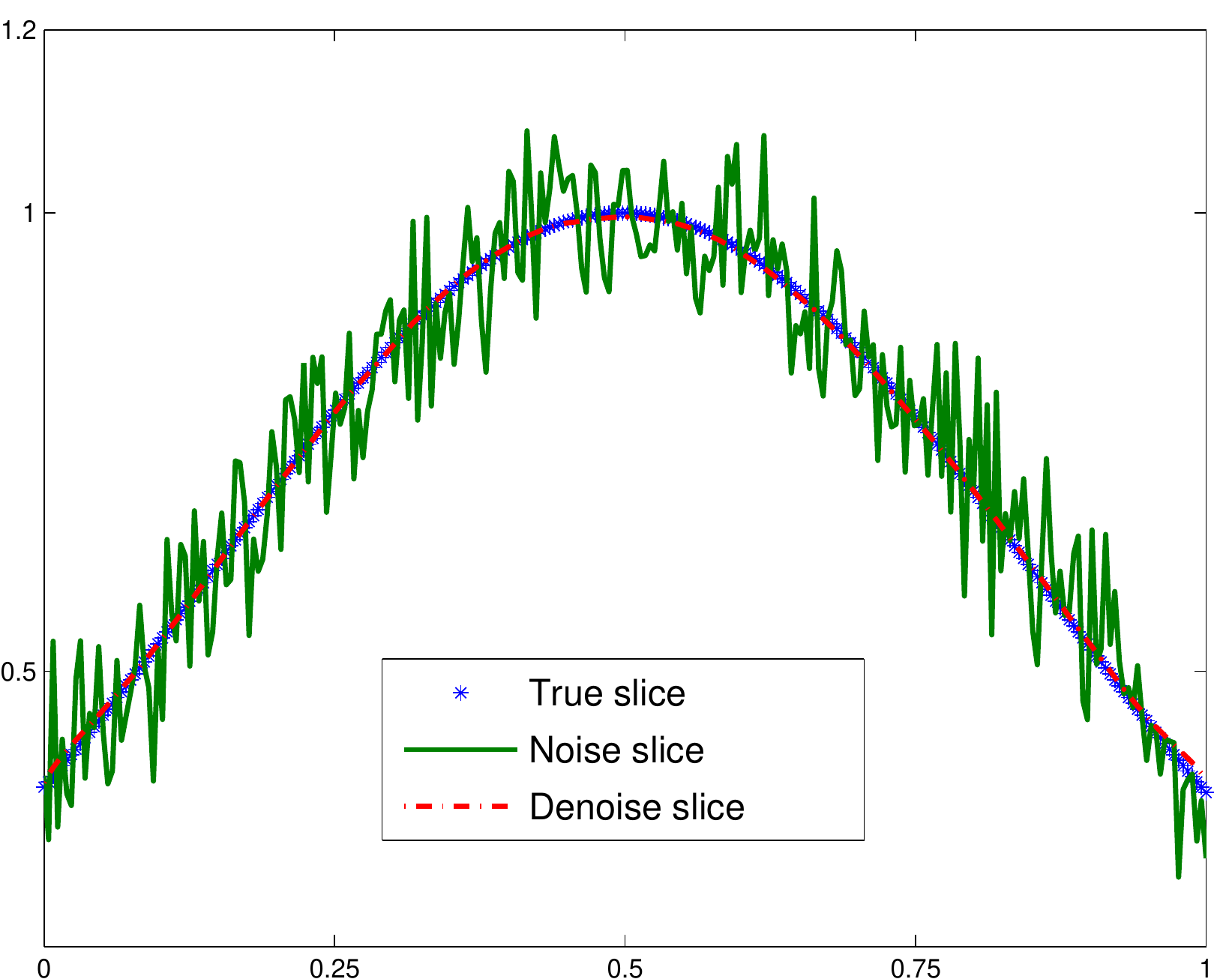}}
\end{center}\caption{Test for \textbf{P1}---Comparisons between treated and Non-treated cases for non-zero boundary conditions ($\delta=\frac{15}{256}$) using PDE-SB. The  treated case has \emph{psnr}=$47.53$ and \emph{snr}=$35.43$, while   the non-treated case has  \emph{psnr}=$23.69$ and \emph{snr}=$10.38$.
Clearly our boundary regularization \S\ref{sec_bndy}
 is effective while direct application of a fractional model leads to incorrect boundary restoration. Here the error $r=\|u-u^*\|_F/\|u^*\|_F$.
}\label{figure-1}
\end{figure}

\subsection{Comparisons of Algorithms 1--4}
In Table \ref{tab0}, we compare the restoration quality (via \emph{psnr} and \emph{snr})
of 4  Algorithms. There, all four test datasets are used. In the cases of synthetic images \textbf{P1} and \textbf{P2} with noise variation $\delta=\frac{10}{255}$, $\lambda$ is taken as $12000$ and  $3800$ respectively and $\alpha=1.6$.
 In the cases of natural images \textbf{P3} and \textbf{P4} with noise variation $\delta=\frac{5}{255}$, $\lambda$ is taken as $18000$ and $20000$ respectively and $\alpha=1.4$.
One can see that, from Table \ref{tab0}, Opti-Nesterov and PDE-SB perform similarly in terms of the best restoration quality (via \emph{psnr} and \emph{snr}).  However in efficiency (computation times \emph{cpu(s)}),
Opti-FISTA and PDE-SB are the best while
Opti-Nesterov takes more computational times than other three algorithms.
Evidently, overall, PDE-SB (Algorithm 1) shows the most consistence in good performance in
tested cases considered.

\begin{table}[tph]\centering
\newsavebox{\tablebox}
\begin{lrbox}{\tablebox}
\begin{tabular}{cccccccccccccccc}
\hline\\
\multicolumn{1}{c}{} &\multicolumn{3}{c}{Opti-FB}&\multicolumn{1}{c}{} &\multicolumn{3}{c}{Opti-Nesterov}&\multicolumn{1}{c}{}
&\multicolumn{3}{c}{Opti-FISTA }&\multicolumn{1}{c}{}&\multicolumn{3}{c}{PDE-SB }\\
   \\
   \cline{2-4}\cline{6-8} \cline{10-12}\cline{14-16}
\\
&\emph{snr}&\emph{psnr}&\emph{cpu(s)}&& \emph{snr}&\emph{psnr}&\emph{cpu(s)}&& \emph{snr}&\emph{psnr}&\emph{cpu(s)}&&\emph{snr}&\emph{psnr}&\emph{cpu(s)}\\
\hline\\
 \textbf{P1} &36.78& 50.09&16.83&&36.91&50.22&27.23&&36.94&50.24&16.71&&36.96&50.27&14.53  \\
 \\
 \textbf{P2} &31.04&53.08&17.28& &31.61&53.67&28.43 & &31.46&53.50&18.14& &31.63&53.69&15.09  \\
  \\
 \textbf{P3}  &29.21&43.29&15.96& &29.40&43.49&16.09& &29.40&43.49&9.75&&29.48&43.56&8.16 \\
 \\
 \textbf{P4}  &25.19&38.01&14.68& &25.35&38.15&16.27& &25.34&38.15&8.62&&25.34&38.14&8.45 \\
 \\
 \hline
\end{tabular}
\end{lrbox}
\caption{Comparisons of optimizing Algorithms, where $\delta=\frac{10}{255}$ for saddle and parabolic surfaces, $\delta=\frac{5}{255}$ for pepper and penguin images.}
\label{tab0}
\scalebox{0.66}{\usebox{\tablebox}}
\end{table}

\subsection{Sensitivity tests for   $\alpha$ and   $\lambda$} 
Since our model (\ref{reg-problem1}) contains two main parameters:
$\alpha$ for the order of differentiation and
$\lambda$ as the coupling parameter for a regularized inverse problem,
it is of interest to test their sensitivity on the restoration quality.
Here we shall test all Algorithms's sensitivity using the image \textbf{P2} - Saddle surface
of size $256\times 256$,
  after adding  zero mean value Gaussian random noise image of range [0, 1] and $\delta=\frac{10}{256}$.

\newBlue{
 Varying $\lambda$ in a large range from $400$ to $60000$,
all four algorithms are tested
on this synthetic image with the results shown in Figs. \ref{fig-Opti_meth_lambda-a} and \ref{fig-Opti_meth_lambda-c} for different stopping criterions (\textbf{GSC:} the general stopping criterions with the relative residual $10^{-4}$, relative error $10^{-8}$, inner iterations 10,\textbf{SSC:} the strong stopping criterions with the relative residual $10^{-7}$, relative error $10^{-10}$, inner iterations 25 ). Different from the TV denoising case where the regularization parameter $\lambda$ is crucial for restoration quality \cite{JPZhang2012}, however, Figs. \ref{fig-Opti_meth_lambda-a} and \ref{fig-Opti_meth_lambda-c} show that
our total $\alpha$-order variation regularization model still obtains a satisfactory solution for a large range of $\lambda$; this is a pleasing
observation. Of course there exists an issue of an optimal choice.

Next varying $\alpha\in (1,2)$ from $1.1$ to $1.9$, Figs. \ref{fig-Opti_meth_alpha-b} and \ref{fig-Opti_meth_alpha-d} show four algorithms's restored results responding to two stopping conditions \textbf{GSC} and \textbf{SSC}. As represented, the smaller $\alpha$  leads to the blocky (staircase) effects in $u$ and the larger $\alpha$ will make solution $u$ too smooth along $x_1$- and $x_2$-directions respectively.
For denoising, our test suggests that $\alpha=1.6$ is suitable for smooth
problems because the diffusion coefficients are almost isotropic in all regions,  leading to   smooth deformation fields, and $\alpha=1.4$ is appropriate
for nonsmooth problems because the diffusion coefficients are close to zero in
regions representing large gradients of the fields, allowing discontinuities at
those regions.

We should emphasize that the stopping criterions have impacted on the actual numerical implementation. In other words, if we drop the limit on the maximal number of inner iterations and relative residuals (and relative errors), some methods take too long but obtain the more satisfactory results.
}

\begin{figure}[!htb]
\begin{center}
\subfigure[\textbf{GSC}:\emph{psnr} vs. $\lambda$]{\label{fig-Opti_meth_lambda-a}
\includegraphics[width=1.3in,height=1.2in]{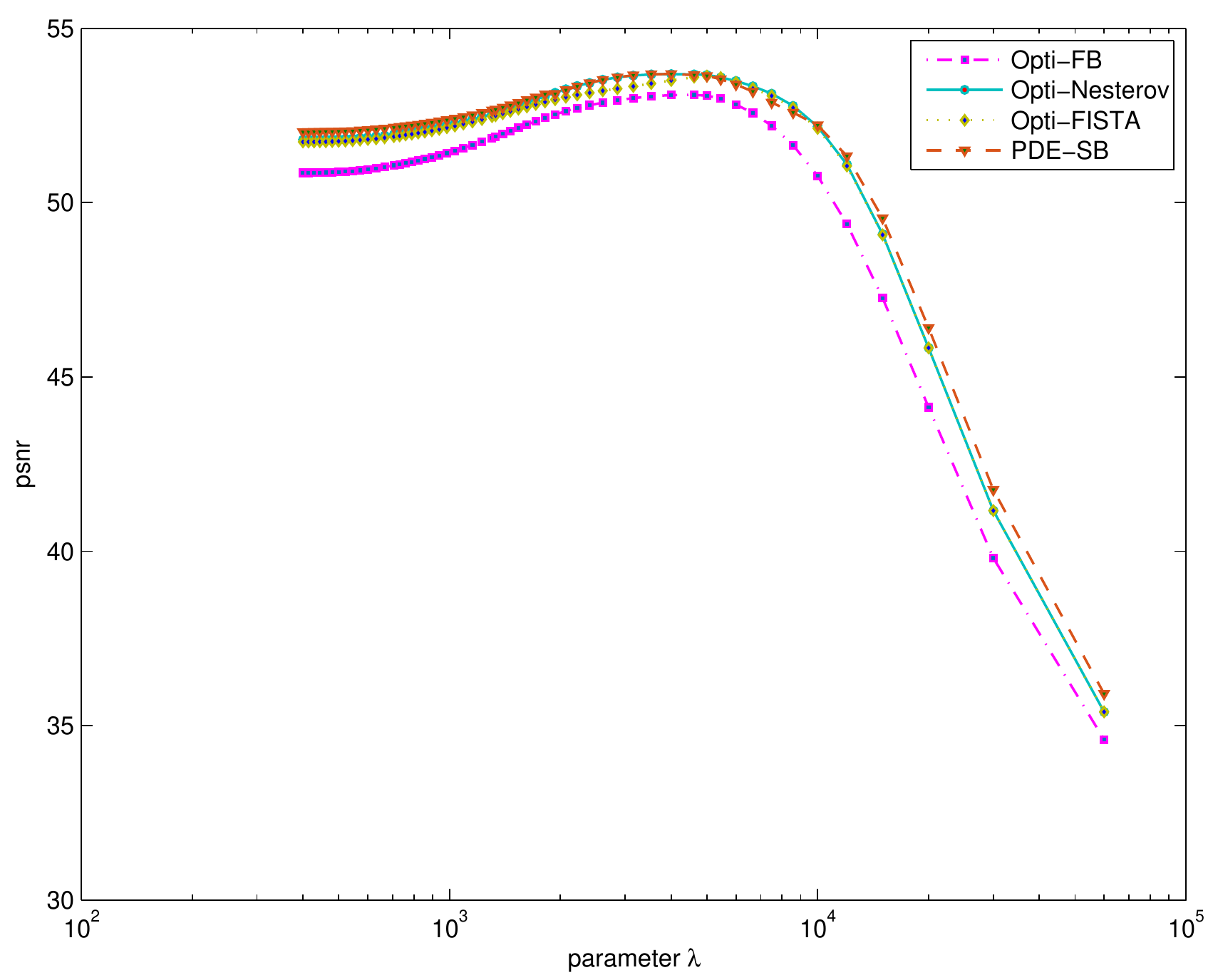}}
\subfigure[\textbf{GSC}:\emph{psnr} vs. $\alpha$]{\label{fig-Opti_meth_alpha-b}
\includegraphics[width=1.3in,height=1.2in]{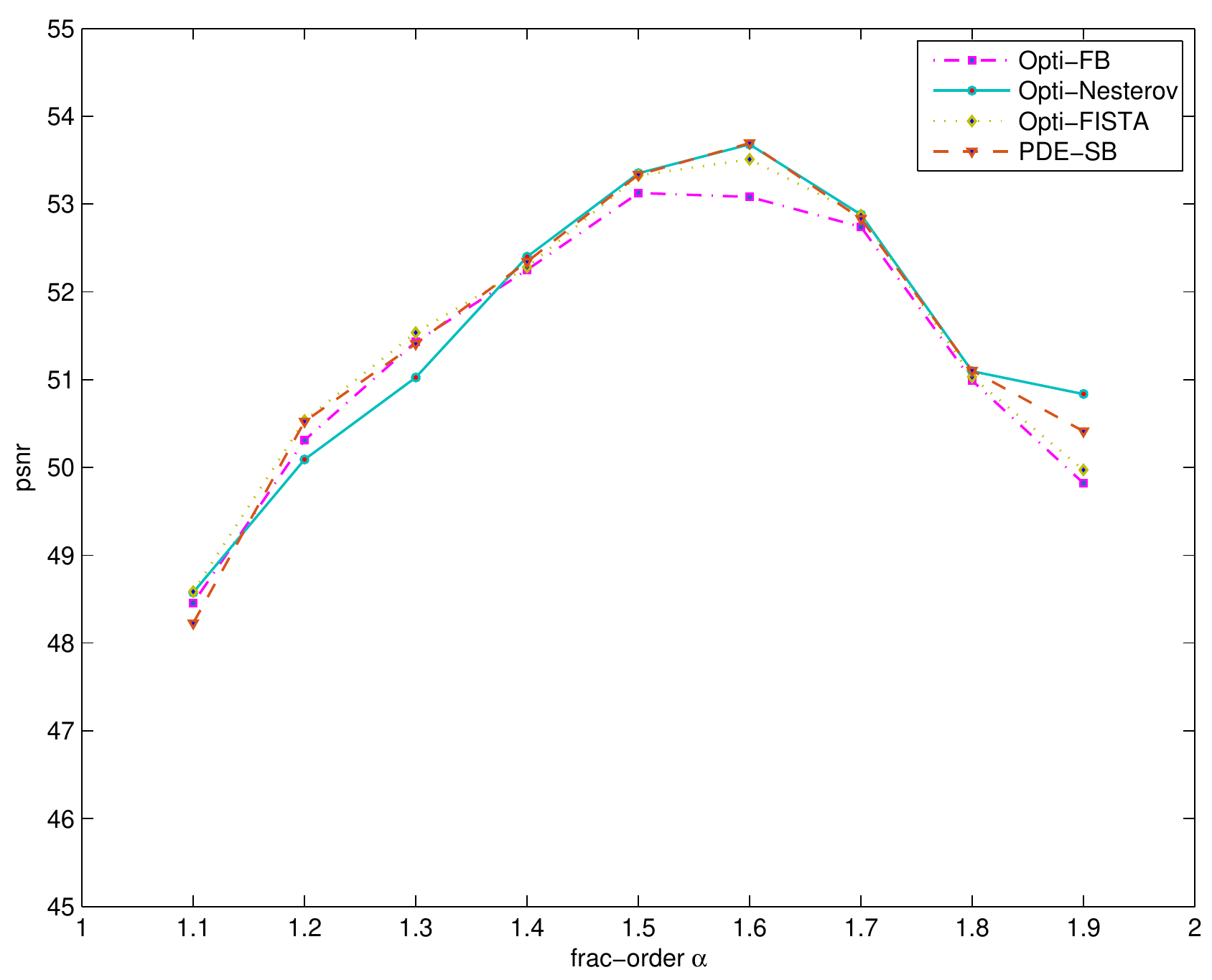}}
\subfigure[\textbf{SSC}:\emph{psnr} vs. $\lambda$]{\label{fig-Opti_meth_lambda-c}
\includegraphics[width=1.28in,height=1.2in]{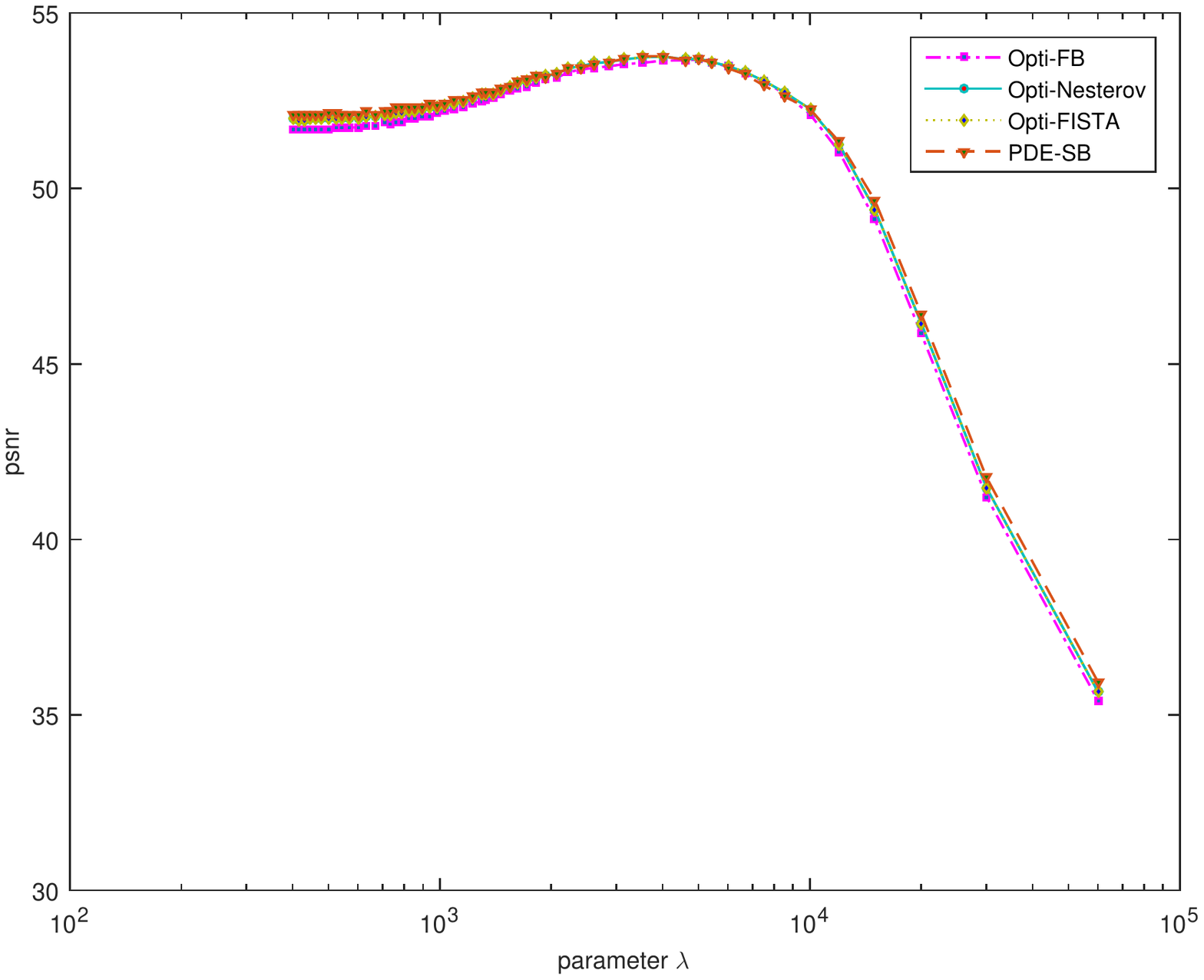}}
\subfigure[\textbf{SSC}:\emph{psnr} vs. $\alpha$]{\label{fig-Opti_meth_alpha-d}
\includegraphics[width=1.28in,height=1.2in]{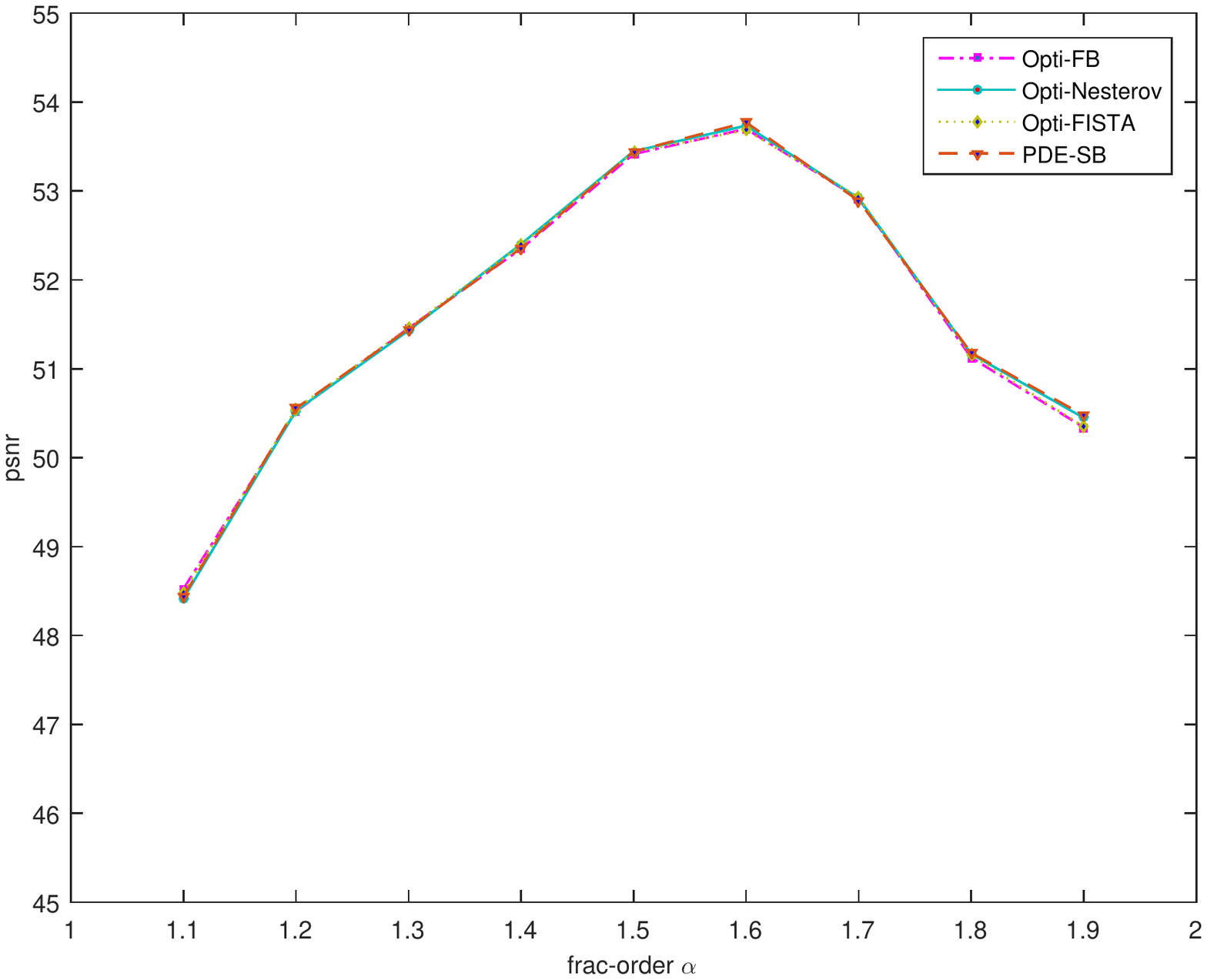}}
\end{center}\caption{Sensitivity test of Algorithms 1-4 to parameters $\lambda$ (with the fixed $\alpha=1.6$) and   $\alpha$ (with the fixed $\lambda=3800$) in the cases of the \textbf{GSC} and \textbf{SSC} stopping conditions.}\label{fig-Opti_meth_lambda}
\end{figure}

\subsection{Comparisons with other non-fractional variational models}
In this test, we compare our total $\alpha$-variation model(PDE-SB) with three popular methods for variational image denoising.
The first compared approach is naturally the TV model proposed by Rudin et al. \cite{LIRudin1992} because the total $\alpha$-order variation model in this work is inspired by it. The second compared work is the mean curvature model \cite{LSun2013} which also addresses the problem of restoring a good result for a smooth image; their approach is different from ours since it is focused on higher order regularization and a multigrid method.  See also
\cite{MLysaker2004,Brito10,WZhu2012}. The third compared approach is the TGV model \cite{KBredies2010} involving a combination of first order and higher-order derivatives to reduce the staircasing effect of the bounded variation functional.

   In Table \ref{tab:3}, we first compare the restoration quality (via \emph{psnr}, \emph{snr})
and efficiency (computation times \emph{cpu(s)}) of four approaches by testing the artificial images (\textbf{P1} - Parabolic surface, \textbf{P2} - saddle surface) and the natural images (\textbf{P3} - Pepper, \textbf{P4} - Penguin); in each approach relevant parameters are shown in
 Table \ref{tab:3}.  \newBlue{We see that, with the emperically optimal parameters $\lambda^*$, the differences of four models are very small, though our new and convex model is slightly better. In other tests  where such optimal parameters are not used, our new model performs more robustly and better.

 In order to present more visual differences, some stronger regularization parameters ($\lambda^*/2$) and higher noise variations (with the noise level $\delta=\frac{30}{255}$) are tested,
 the solution's visual representations restoring the natural image
 \textbf{P3} - Pepper  in Fig. \ref{fig4-b}  are shown in Fig. \ref{figure-4}. While ROF denoising leads to   blocky results, the mean curvature model
performs better in
the smooth regions but exhibits more smooth near discontinuities,
the total generalized variation model
leads to further improvements over the aforementioned models.
 The
total fractional-order variation model leads to significantly better results.
The reason is that the new model tries
to approximate the image based on affine functions or non-local high order smooth functions, which is clearly better in this
case, in other words, our approach is more effective in eliminating the noise for smooth images and is competitive to high order
methods; in efficiency the new approach (PDE-SB) is much faster than the TGV and the
mean curvature. We also plot four error results between the restored and true images along a diagonal (magenta) line in Fig. \ref{fig4-a} for comparison in Fig. \ref{figure-4-2}; we see that PDE-SB
  produces the best restored surface, which show a major advantage (or better performance) of using our total $\alpha$-order variation model (\ref{reg-problem1}) when the test image is smooth, and even when the contrast between meaningful objects and the background is low.}

\begin{table}[!h]\centering\scriptsize
\begin{lrbox}{\tablebox}
\begin{tabular}{cccccccccccccc}
\hline \\
\multicolumn{1}{c}{} &\multicolumn{1}{c}{} &\multicolumn{3}{c}{Mean curvature \cite{LSun2013}} &
\multicolumn{3}{c}{TV \cite{LIRudin1992}}&\multicolumn{3}{c}{TGV \cite{KBredies2010}}&
\multicolumn{3}{c}{Total $\alpha$-order model (\ref{reg-problem1})}\\
\\
\cline{3-5}\cline{7-8}\cline{10-11}\cline{13-14} \multicolumn{1}{c}{}
\\
&$\hat{\delta}$&\emph{}&\emph{snr}&\emph{psnr}&\emph{}&\emph{snr}&\emph{psnr}&\emph{}&\emph{snr}&\emph{psnr}&\emph{}&\emph{snr}&\emph{psnr}\
   \\
   \hline
 \\
               &$10$& &33.44&46.74 & &32.17& 45.52      & &36.41&   49.72&                          & 37.55&50.86\\
     \textbf{P1}       &$20$ & &30.19&43.50 & &29.55& 42.83       & &33.03&   46.33&                     & 33.52&46.83\\
 \multicolumn{1}{c}{} &\multicolumn{1}{c}{} &&\multicolumn{2}{c}{$\lambda_1=1/0.4\times 256^2$} & &\multicolumn{2}{c}{$\lambda_1=1026$} &&
\multicolumn{2}{c}{$\lambda_1=1/1.2\times 256^2$}&&
\multicolumn{2}{c}{$\lambda^{1D}_1=0.1$, $\lambda_1=21900$}\\
\multicolumn{1}{c}{} &\multicolumn{1}{c}{Para}& &\multicolumn{2}{c}{$\lambda_2=1/0.03\times 256^2$} &&  \multicolumn{2}{c}{$\lambda_2=535$} &&
\multicolumn{2}{c}{$\lambda_2=1/0.6\times 256^2$}&&
\multicolumn{2}{c}{$\lambda^{1D}_2=0.1$, $\lambda_2=14400$}\\
\\
           &$10$& &27.27&49.31         & &23.09&    45.13&      &30.75& 52.68  &                  & 32.02&54.18\\
 \textbf{P2}&$20$ & &22.88&44.92           & &19.45&   41.49&      &25.62& 47.51&               & 26.48&48.54\\
 \multicolumn{1}{c}{} &\multicolumn{1}{c}{} &&\multicolumn{2}{c}{$\lambda_1=1/0.9\times 256^2$} &&
\multicolumn{2}{c}{$\lambda_1=883$}& & \multicolumn{2}{c}{$\lambda_1=1/0.9\times 256^2$} & &
\multicolumn{2}{c}{$\lambda^{1D}_1=1$, $\lambda_1=1800$}\\
\multicolumn{1}{c}{} &\multicolumn{1}{c}{Para} &&\multicolumn{2}{c}{$\lambda_2=1/0.01\times 256^2$} &&
\multicolumn{2}{c}{$\lambda_2=488$}&  &  \multicolumn{2}{c}{$\lambda_2=1/0.5\times 256^2$} &&
\multicolumn{2}{c}{$\lambda^{1D}_2=0.2$, $\lambda_2=1800$}\\
 \\

            &$10$& &20.43&38.80          & &20.08&38.35&   &20.40&38.78&       &20.48&38.86\\
             &$15$ & &18.76&37.11        & &18.01&36.69&    &18.68&37.12 &     &18.84&37.20\\
 \textbf{P3} &$20$&
               &17.48&35.82                  & &17.17&35.33&     &17.55&35.87&            &17.57&35.90\\
 \multicolumn{1}{c}{} &\multicolumn{1}{c}{} &&\multicolumn{2}{c}{$\lambda_1=1/16\times 256^2$} &&
\multicolumn{2}{c}{$\lambda_1=2216$}& &\multicolumn{2}{c}{$\lambda_1=1/55\times 256^2$} &
&\multicolumn{2}{c}{$\lambda^{1D}_1=1$, $\lambda_1=16500$}\\
\multicolumn{1}{c}{} &\multicolumn{1}{c}{Para} &&\multicolumn{2}{c}{$\lambda_2=1/14\times 256^2$} &
&\multicolumn{2}{c}{$\lambda_2=1373$}& &\multicolumn{2}{c}{$\lambda_2=1/26\times 256^2$} &
&\multicolumn{2}{c}{$\lambda^{1D}_2=0.1$, $\lambda_2=9300$}\\
\multicolumn{1}{c}{} &\multicolumn{1}{c}{} &&\multicolumn{2}{c}{$\lambda_3=1/6\times 256^2$} &
&\multicolumn{2}{c}{$\lambda_3=893$}&  &\multicolumn{2}{c}{$\lambda_3=1/12\times 256^2$} &
&\multicolumn{2}{c}{$\lambda^{1D}_3=0.01$, $\lambda_3=6200$}\\
 \\
                                    &$5$& &25.16&37.95  & &24.85&37.58     & &25.39&38.20     & &25.34&38.14\\
                            &$10$ & &21.72&34.60 & &21.33& 34.07   & &21.82&34.71&     &21.75&34.62\\
 \textbf{P4}&$15$& &19.26&32.05&  &18.66&31.29&        &19.44&32.21& &19.42&32.20\\
   \multicolumn{1}{c}{} &\multicolumn{1}{c}{} &&\multicolumn{2}{c}{$\lambda_1=1/9\times 256^2$} &  &\multicolumn{2}{c}{$\lambda_1=3341$} &
&\multicolumn{2}{c}{$\lambda_1=1/49\times 256^2$}&
&\multicolumn{2}{c}{$\lambda^{1D}_1=0.1$, $\lambda_1=24000$}\\
\multicolumn{1}{c}{} &\multicolumn{1}{c}{Para} &&\multicolumn{2}{c}{$\lambda_2=1/5\times 256^2$} &  &\multicolumn{2}{c}{$\lambda_2=1856$} &
&\multicolumn{2}{c}{$\lambda_2=1/20\times 256^2$}&
&\multicolumn{2}{c}{$\lambda^{1D}_2=0.1$, $\lambda_2=8000$}\\
\multicolumn{1}{c}{} &\multicolumn{1}{c}{} &&\multicolumn{2}{c}{$\lambda_3=1/6\times 256^2$} &  &\multicolumn{2}{c}{$\lambda_3=1095$} &
&\multicolumn{2}{c}{$\lambda_3=1/11\times 256^2$}&
&\multicolumn{2}{c}{$\lambda^{1D}_3=0.1$, $\lambda_3=18500$}\\
 \hline
\end{tabular}
\end{lrbox}
\caption{Comparisons of four models in restoration quality: the total $\alpha$ variation model (\ref{reg-problem1}), Mean Curvature \cite{MLysaker2004,Brito10,LSun2013}, TV \cite{LIRudin1992, RPFedkiw2003,TFChan2005b} and TGV \cite{KBredies2010} models for synthetic images (\textbf{P1}   and \textbf{P2} in Fig. \ref{figexample}) and natural images (\textbf{P3} in Fig. \ref{fig4-a})
with different noise variances $\delta_j=\frac{\hat{\delta}_j}{255}$ (correspondingly we use $\lambda_j$). We first fix $\mu=1.1$, $\gamma=1$, $\alpha=1.6$ for \textbf{P1-P2}, $\alpha=1.15$ and $\alpha=1.1$ for \textbf{P3-P4} respectively in the total $\alpha$-order variation model, and $\gamma=19$, $\beta=10^{-5}$ in the mean curvature model, and two weight parameters of the first and second order term in TGV model ($\nu_0=1$, $\nu_1=2$ for \textbf{P1-P2}, $\nu_0=1$, $\nu_1=0.5$ for \textbf{P3-P4}), other parameters are as shown on
the \emph{``para"} rows. $\lambda^{1D}$ from (\ref{eq_1D}) is required by the new model only.
}
\scalebox{0.85}{\usebox{\tablebox}}\label{tab:3}
\end{table}

\begin{figure}[!h]
\begin{center}
\subfigure[True $z$.]{\label{fig4-a}
\includegraphics[width=1.2in,height=1.2in]{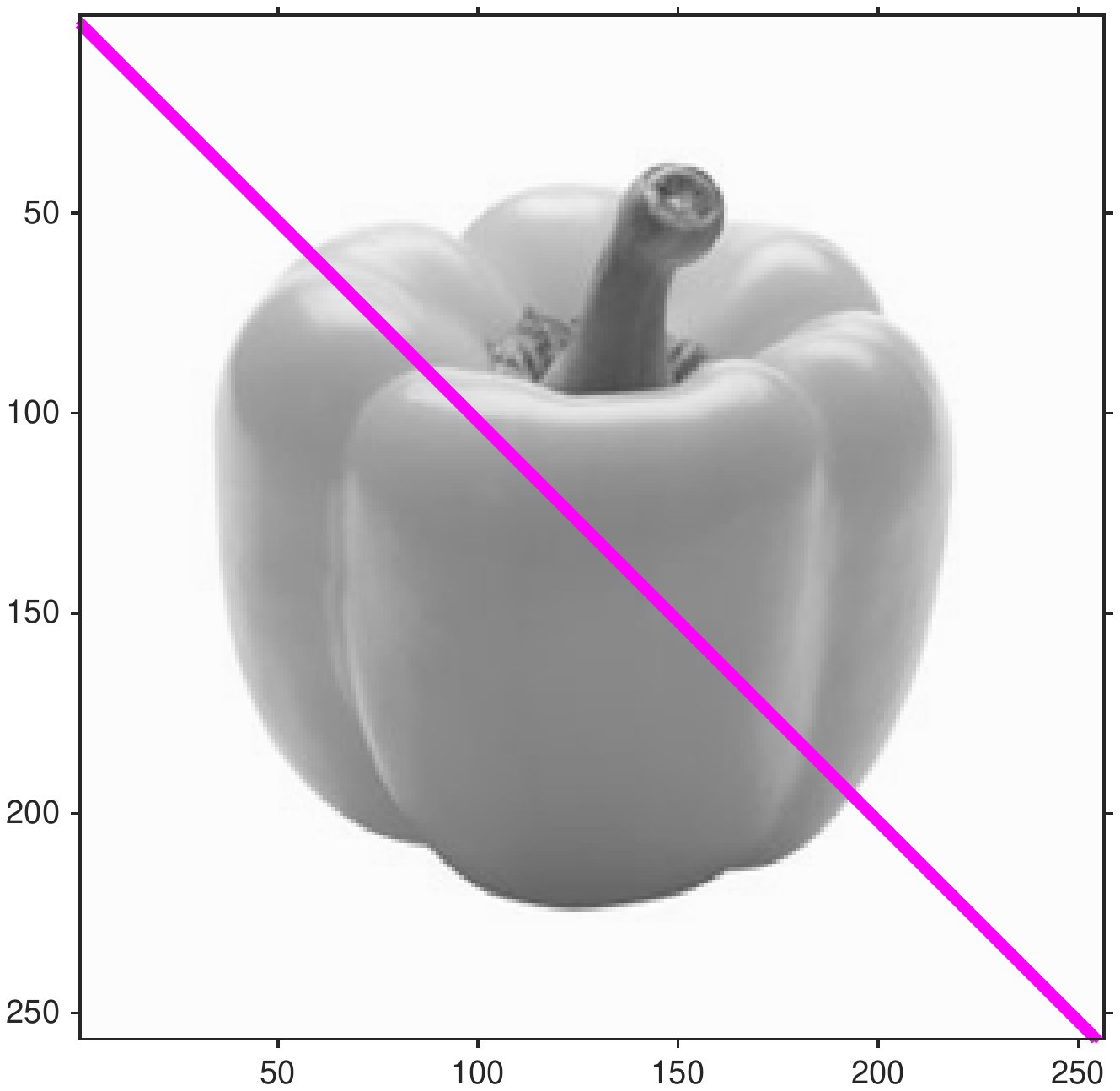}}\;\;\;\;
\subfigure[Noise $f$.]{\label{fig4-b}
\includegraphics[width=1.2in,height=1.2in]{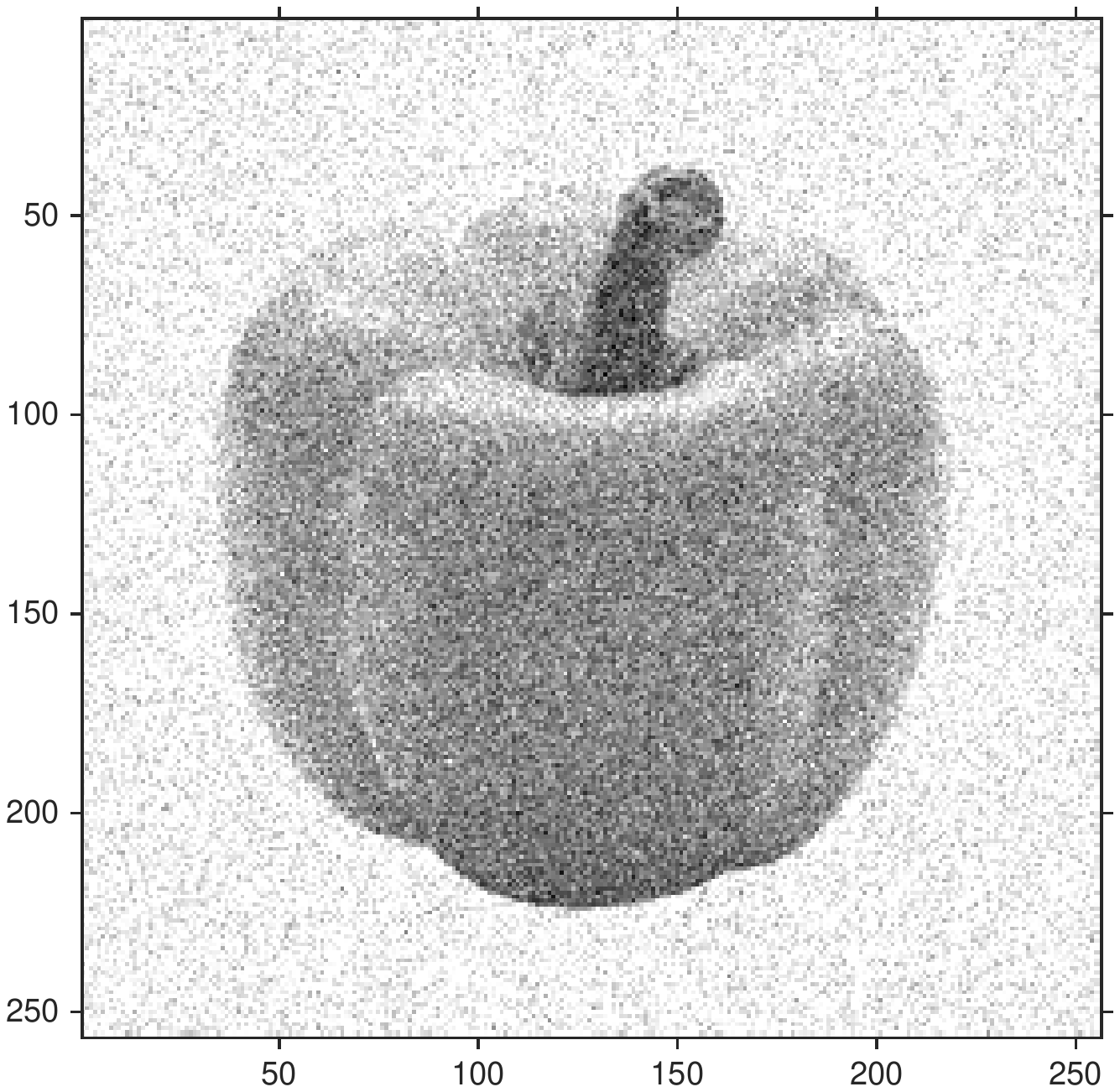}}\;\;\;\;
\subfigure[TV.]{\label{fig4-d}
\includegraphics[width=1.2in,height=1.2in]{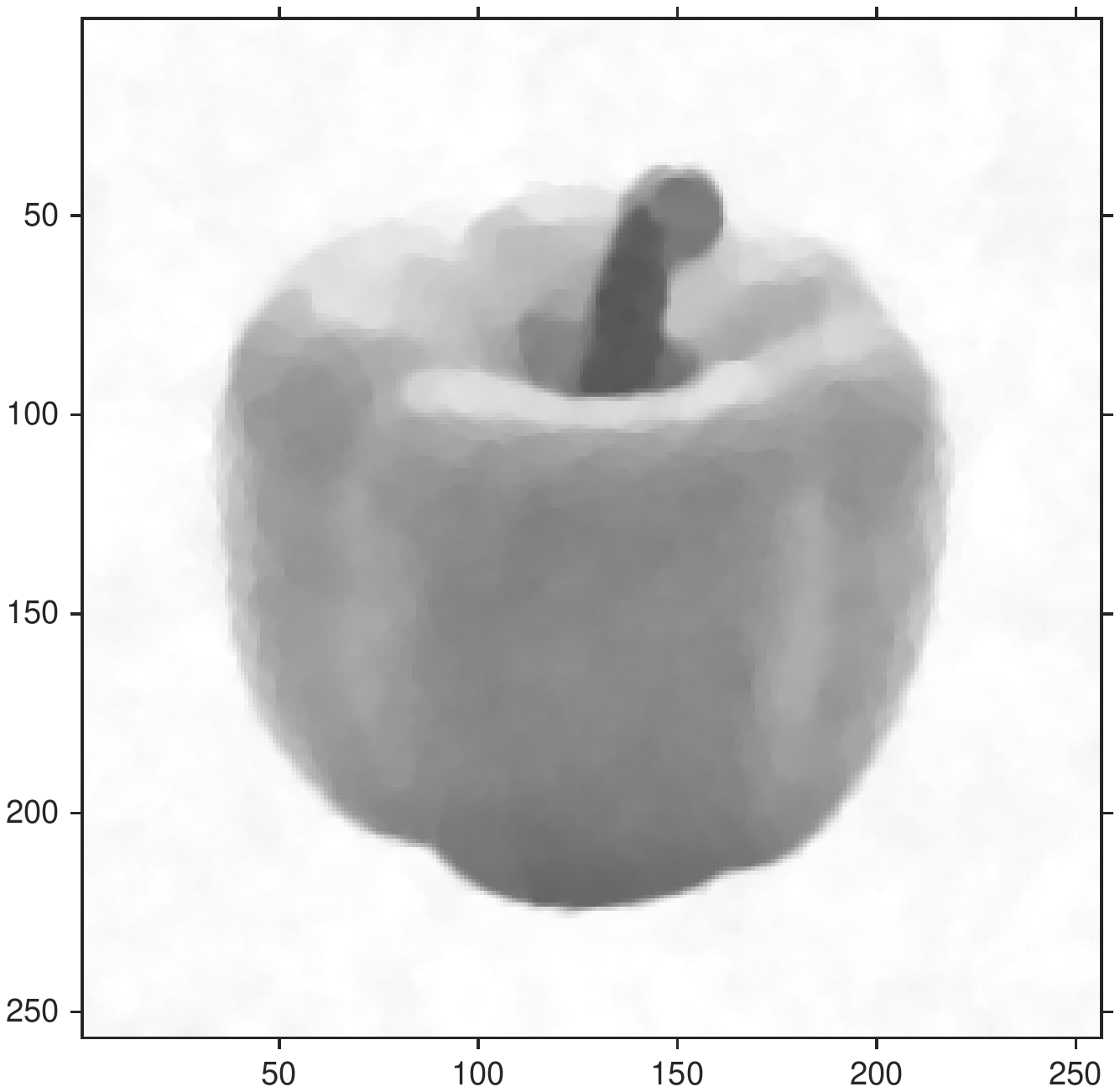}}\\
\subfigure[Mean Curvature.]{\label{fig4-c}
\includegraphics[width=1.2in,height=1.2in]{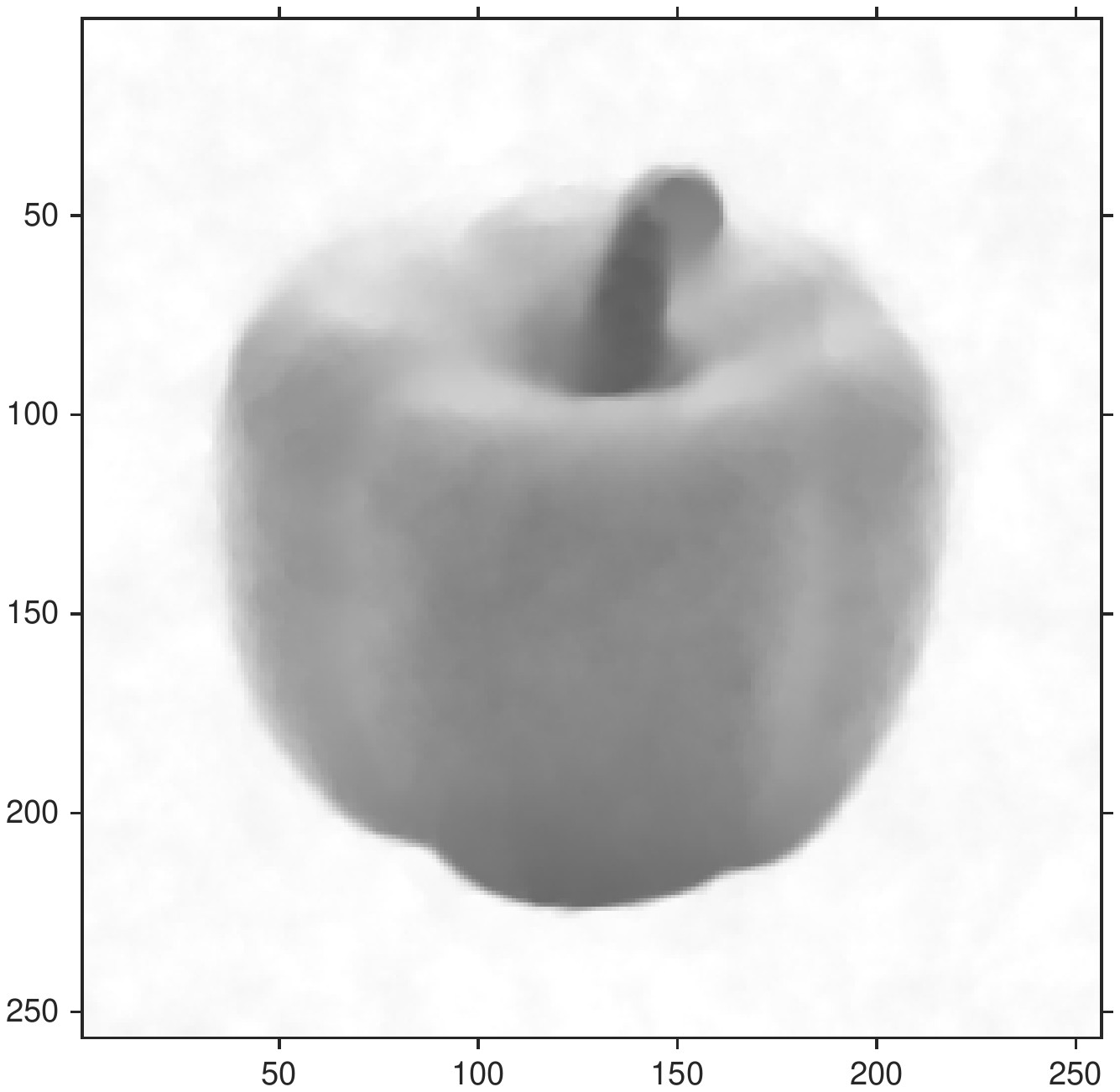}}\;\;\;\;
\subfigure[TGV.]{\label{fig4-f}
\includegraphics[width=1.2in,height=1.2in]{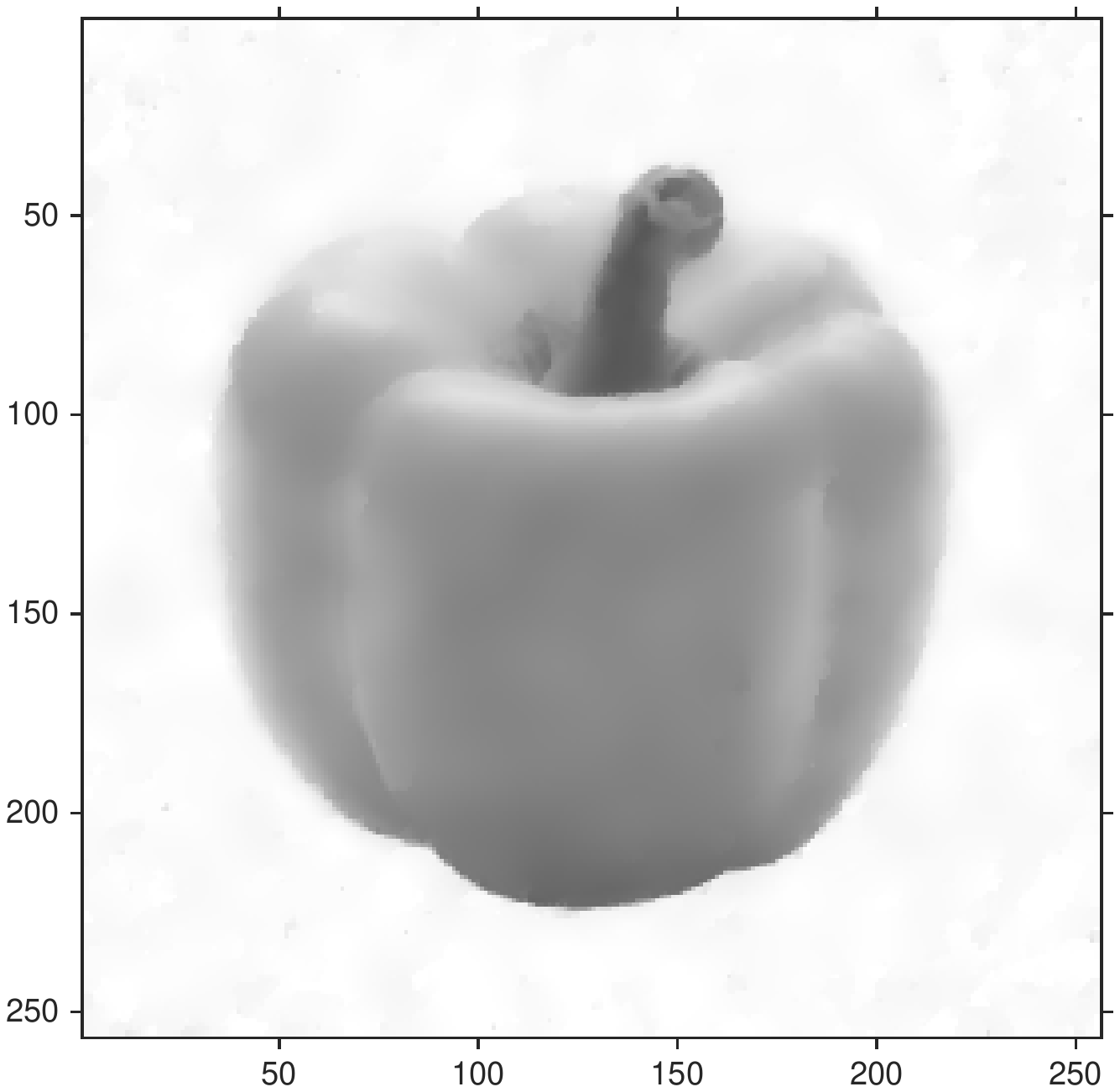}}\;\;\;\;
\subfigure[Our Approach.]{\label{fig4-e}
\includegraphics[width=1.2in,height=1.2in]{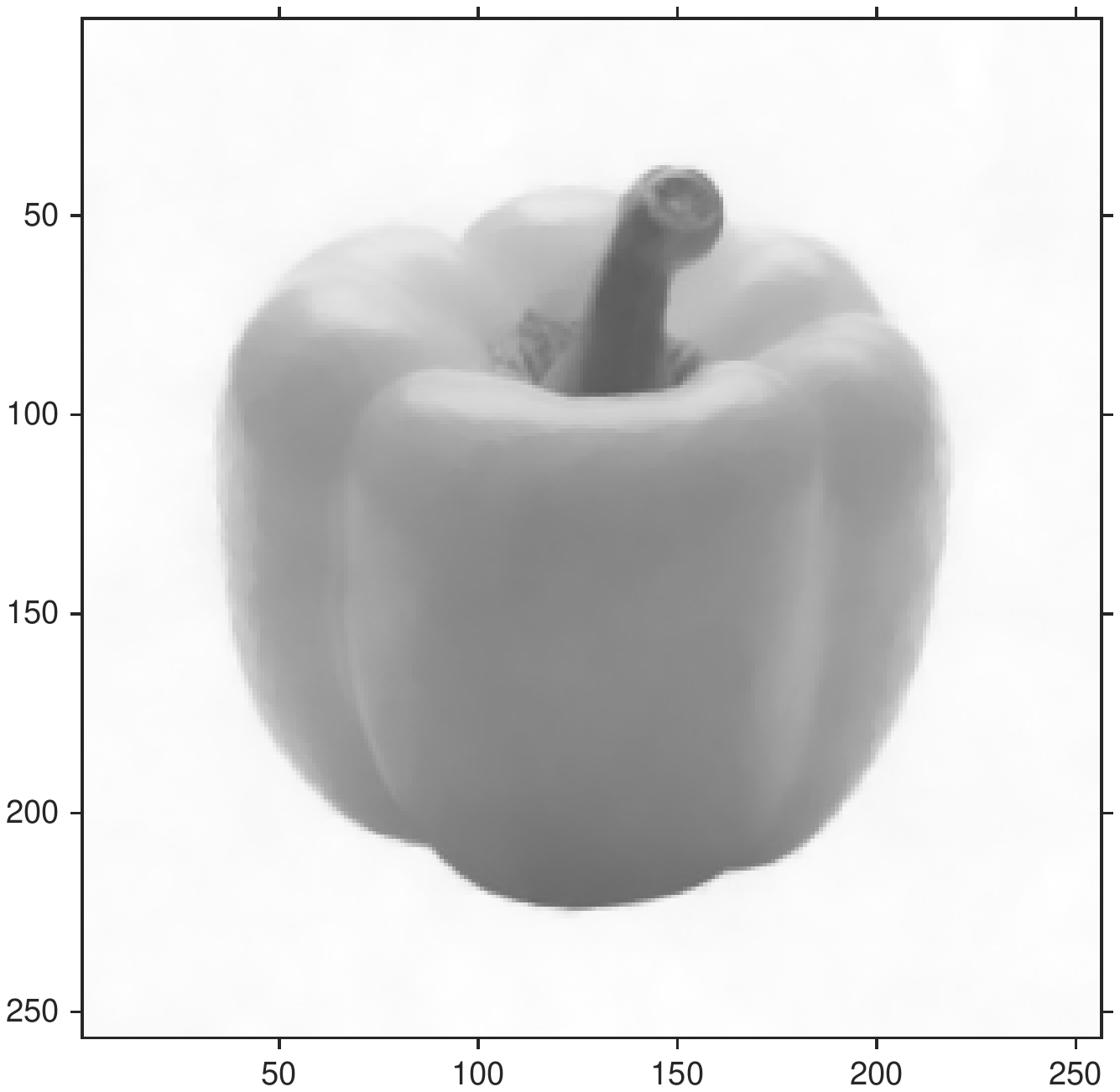}}
\end{center}
\caption{Comparison I ---Comparisons of our PDE-SB with TV, mean curvature and TGV models.}\label{figure-4}
\end{figure}

\begin{figure}[!h]
\begin{center}
\subfigure[$u$]{\label{fig4-2a}
\includegraphics[width=4.5in,height=2.5in]
{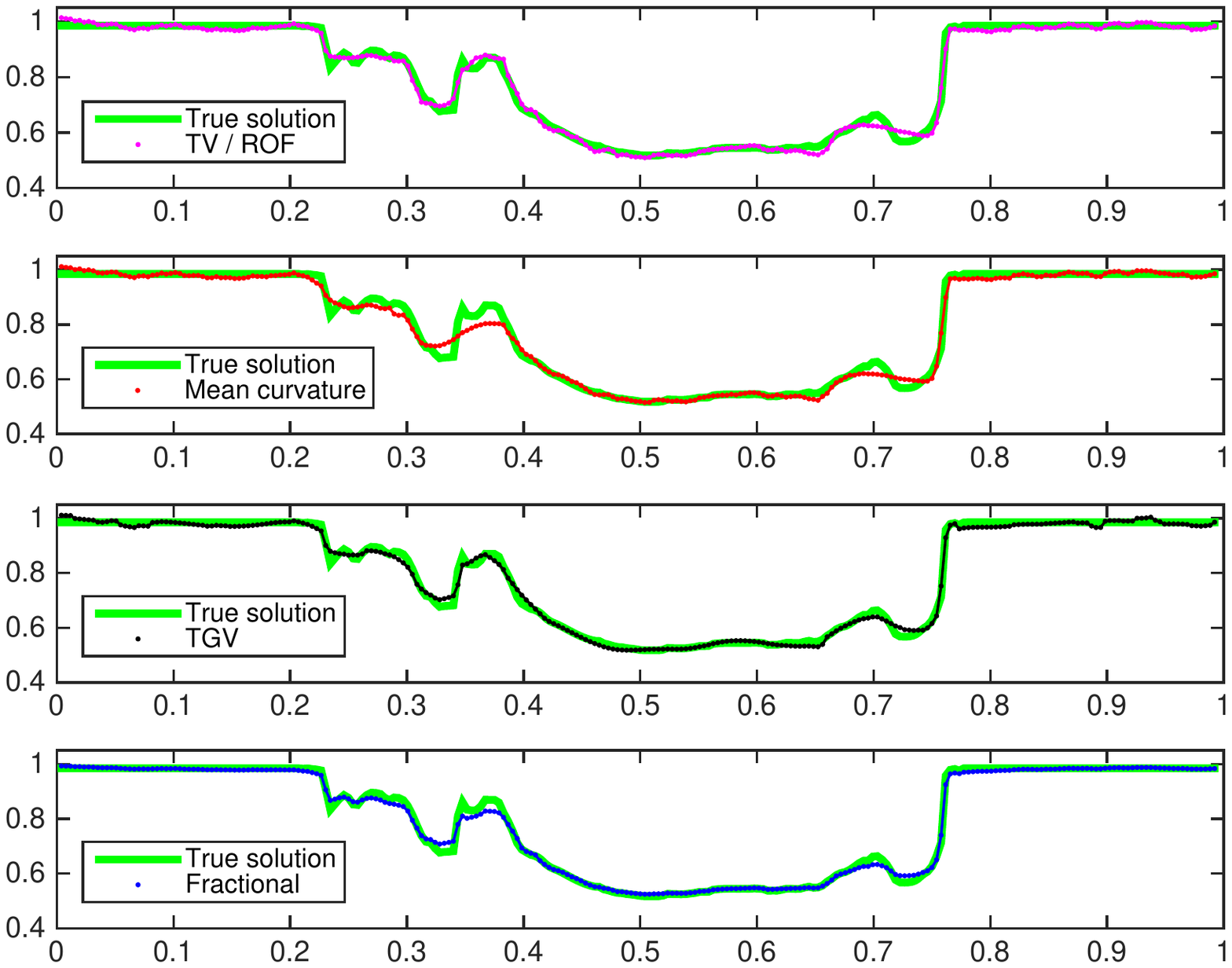}}\\
\subfigure[error $u-u^*$]{\label{fig4-2b}
\includegraphics[width=4.5in,height=2.5in]
{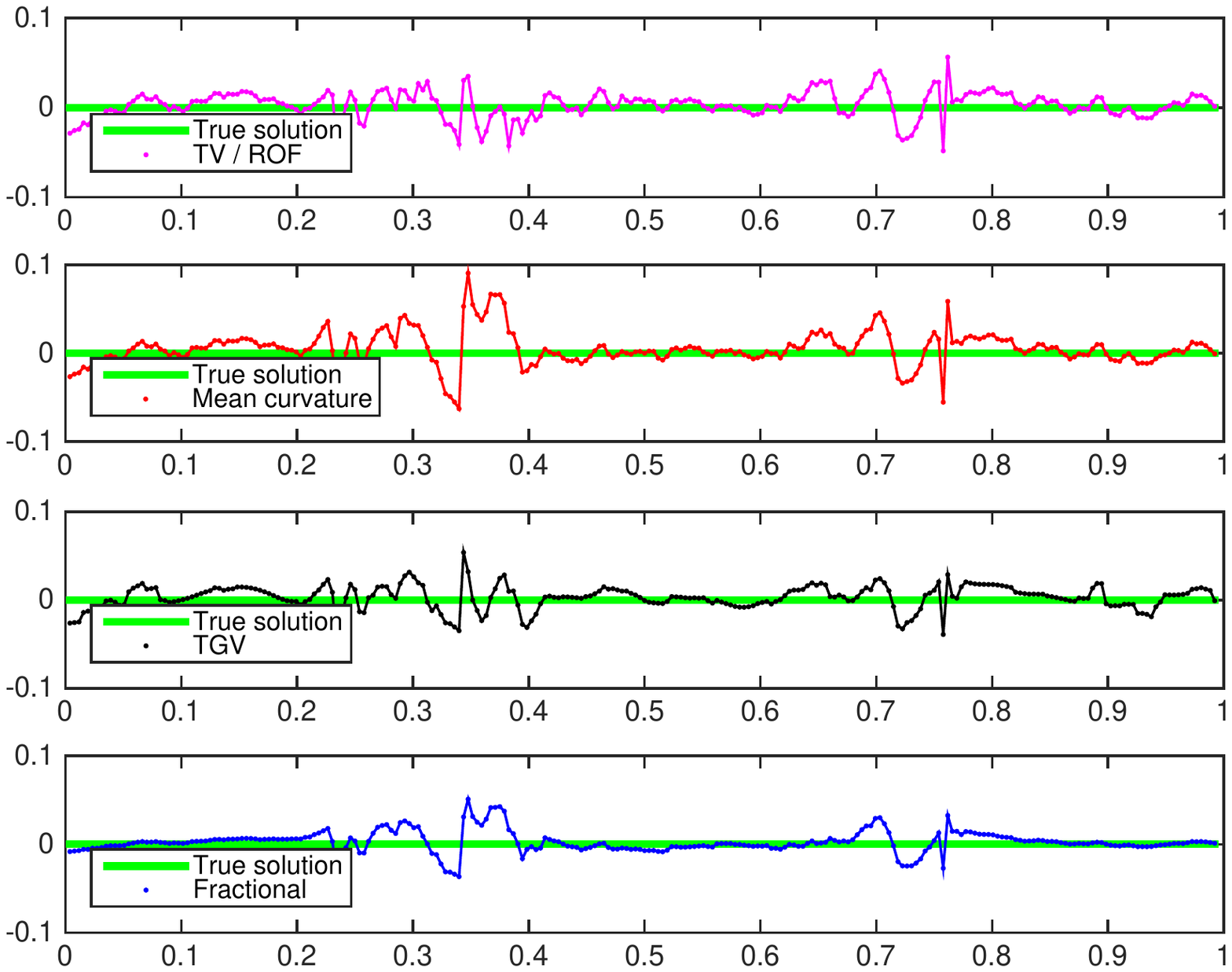}}
\end{center}
\caption{Comparison I --- The slice presentations of four restorations along a diagonal line in Fig. \ref{fig4-a}.}\label{figure-4-2}
\end{figure}

\section{Conclusions}
The total  $\alpha$-order variation regularization with fractional order derivative is potentially useful in modeling all imaging problems.
In this paper we analyzed rigorously a simple variational model using
total $\alpha$-order variation for image denoising.
One Split-Bregman based algorithm  and three   optimization-based algorithms were developed to solve the resulting image inverse problem. Instead of using the usual fixed and zero boundary conditions, we proposed a boundary regularization method to treat the fractional order derivatives.
Numerical results show that the PDE-based Split-Bregman algorithm (PDE-SB) performs
  similarly to (though more stably than) optimization-based approaches
while
 our boundary  regularization method is essential for getting
 good results for imaging denoising. Moreover,
  PDE-SB outperforms currently competitive variational models in terms of restoration quality.
 There are still outstanding issues with
our proposed model and algorithms; among others optimal selection of $\lambda$ is to be addressed.
 Future work will also consider generalization of this work to other image inverse problems.

\subsection*{Appendix: Proof of Theorem \ref{reg-problem02}}
To shorten the proof, let $\omega$ be a function in $W_1^\alpha(\Omega)$ to be specified shortly.
For $u\in W_1^\alpha(\Omega)\subset\text{BV}^\alpha(\Omega)$, we compute
the first-order G-derivative (Gateaux) of the functional
$J(u)$  in the direction $\omega$ by
\begin{equation}\label{Euler-Lagrange1}
\begin{split}
J'(u)\omega&=\lim\limits_{t\rightarrow 0}\frac{J(u+t\omega)-J(u)}{t}\
 =\lim\limits_{t\rightarrow 0}\frac{Q(u+t\omega)-Q(u)}{t}+
\frac{\lambda}{2}\frac{F(u+t\omega)-F(u)}{t}
\end{split}
\end{equation}
where $Q(u)=\frac{\mu}{2}\int_\Omega|\dd-\nabla^\alpha u+\frac{\pp}{\mu}|^2dx$ -- see (\ref{eqn_u2}).
Using the Taylor series w.r.t $t$ yields
\begin{equation}\label{Euler-Lagrange2}
J'(u)\omega= \int_\Omega \boldsymbol{W}\cdot\nabla^\alpha\omega dx+\lambda\int_\Omega (u-z)\;\omega dx
\end{equation}
with $\boldsymbol{W}=-\mu(\dd-\nabla^\alpha u+\frac{\pp}{\mu})$.
Recall that
\begin{equation}\label{Euler-Lagrange4}
\begin{split}
\int_\Omega \boldsymbol{W}\cdot\nabla^\alpha\omega dx
=
(-1)^n\int_\Omega \omega {}^{C}\diver^\alpha\boldsymbol{W} dx
- &\sum_{j=0}^{n-1}(-1)^j
\int_0^1  D^{\alpha-n+j}_{[a,b]}W_1\frac{\partial^{n-j-1}\omega(x)}{\partial x_1^{n-j-1}}\Big|_{x_1=0}^{x_1=1} dx_2\\
- &\sum_{j=0}^{n-1}(-1)^j
\int_0^1  D^{\alpha-n+j}_{[c,d]}W_2\frac{\partial^{n-j-1}\omega(x)}{\partial x_2^{n-j-1}}\Big|_{x_2=0}^{x_2=1} dx_1.
\end{split}
\end{equation}
where we note $n=2$ for $1<\alpha<2$. Next consider 2 case  studies.

\noindent
i). Given $u(x)\big|_{\partial\Omega}=b_1(x), \ \text{and }\
    \frac{\partial u(x)}{\partial n}\Big|_{\partial\Omega}=b_2(x)$,
    since $\big(u(x)+t\omega(x)\big)\big|_{\partial\Omega}
    =\big(u(x)\big)\big|_{\partial\Omega}=b_1(x)$ and $\frac{\partial \big(u(x)+t\omega(x)\big)}{\partial n}\Big|_{\partial\Omega}=\frac{\partial u(x)}{\partial n}\big|_{\partial\Omega}=b_2(x)$, it suffices to take $\omega\in\mathscr{C}_0^1(\Omega,\mathbb{R})$. Such a choice ensures
    $\frac{\partial^i \omega(x)}{\partial n^i}\Big|_{\partial\Omega}=0,
    i=0,1 \ \
    \Rightarrow \ \
    \frac{\partial^{n-j-1}\omega(x)}{\partial x_1^{n-j-1}}
    \Big|_{x_1=0  \ \mbox{or} \ 1} = \frac{\partial^{n-j-1}\omega(x)}{\partial x_1^{n-j-1}}
    \Big|_{x_2=0  \ \mbox{or} \ 1} = 0,\ n-j-1=0,1$. Hence
    equation (\ref{Euler-Lagrange1}) with (\ref{Euler-Lagrange2}) reduces to (\ref{fractionalSBEL}).

\noindent
ii). Keep $\omega\in W_1^\alpha(\Omega)$. Since $\frac{\partial^{n-j-1}\omega(x)}{\partial x_1^{n-j-1}}
    \Big|_{x_1=0  \ \mbox{or} \ 1} \not=0,\
     \frac{\partial^{n-j-1}\omega(x)}{\partial x_1^{n-j-1}}
    \Big|_{x_2=0  \ \mbox{or} \ 1} \not=0$, the boundary terms in equation (\ref{Euler-Lagrange4})
can only diminish if
    $$  D^{\alpha-n+j}_{[a,b]}W_1 \Big|_{x_1=0  \ \mbox{or} \ 1}
    = 0 \ \ \mbox{and} \ \
     D^{\alpha-n+j}_{[c,d]}W_2 \Big|_{x_2=0  \ \mbox{or} \ 1}
    = 0 \ \
    \Rightarrow \ \  D^{\alpha-n+j} \boldsymbol{W}\cdot n = 0, j=0,1.$$
The proof is complete. 

\begin{remark}
In imaging applications, the above first set i) of boundary conditions seems not reasonable, because one hardly knows a priori what $b_1, b_2$ should be.
The second set ii) of boundary conditions appears complicated which might be simplified as follows.\

From \cite[Section 2.3.6 pp.75]{IPodlubny1999}, if $W_1(x)$ has a sufficient number of continuous derivatives, then\\
$  D^{\alpha-n+j}_{[0,\;1]}W_1 \Big|_{x_1=0  \ \mbox{or} \ 1}
    = 0$ for any $\alpha\in (1,2)$ is equivalent to $\frac{\partial^j W_1}{\partial x_1^j} \Big|_{x_1=0  \ \mbox{or} \ 1}=0\;( j=0,1)$, i.e.,
    $$W_1\Big|_{x_1=0  \ \mbox{or} \ 1}=0\;\; \mbox{and}\;\; \frac{\partial W_1}{\partial x_1} \Big|_{x_1=0  \ \mbox{or} \ 1}=0.$$
    Indeed, if the $n$-th derivative of $u(x)$ is integrable in $[0,1]$, then $W_1\Big|_{x_1=0  \ \mbox{or} \ 1}=0$ is equivalent  to 
    $$u(x)\Big|_{x_1=0  \ \mbox{or} \ 1}=0\;\; \mbox{and}\;\;\frac{\partial u(x)}{\partial x_1} \Big|_{x_1=0  \ \mbox{or} \ 1}=0;$$
    on the other hand, $\frac{\partial^k u(x)}{\partial x_1^k} \Big|_{x_1=0  \ \mbox{or} \ 1}=0$ (for all $k=0,1,2$) are equivalent to $\frac{\partial^\alpha u(x)}{\partial x_1^\alpha} \Big|_{x_1=0  \ \mbox{or} \ 1}=0$ and $\frac{\partial^{1+\alpha} u(x)}{\partial x_1^{1+\alpha}} \Big|_{x_1=0  \ \mbox{or} \ 1}=0$, hence one has $\frac{\partial W_1}{\partial x_1} \Big|_{x_1=0  \ \mbox{or} \ 1}=0$. The derivations of $W_2$ are similar to those of $W_1$.
\end{remark}

{\bf For related papers or Matlab codes, see Authors'   page: \ \ \url{http:\\www.liv.ac.uk/~cmchenke}}

\end{document}